\documentclass[12pt]{article}

\usepackage{ltexpprt}
\usepackage{algorithmicx}
\usepackage{algorithm}
\usepackage{mathtools}
\usepackage{algpseudocode}

\usepackage{color}
\usepackage{verbatim}

\usepackage{xcolor}

\usepackage{varwidth}
\usepackage[numbers]{natbib}

\def\for{\hbox{ for }}

\def\Var{\mbox{Var}}
\def\Max{\mbox{Max}}
\def\Min{\mbox{Min}}
\def\prob{\hbox{Pr}}
\def\proj{\mbox{Proj}}
\def\Proj{\mbox{Proj}}

\def\prob{\mbox{Pr}}

\newtheorem{claim}{Claim}[section]
\newtheorem{assum}{Assumption}

\newenvironment{remark}[1][Remark:]{\begin{trivlist}
\item[\hskip \labelsep {\bfseries #1}]}{\end{trivlist}}

\newcommand{\qed}{\nobreak \ifvmode \relax \else
      \ifdim\lastskip<1.5em \hskip-\lastskip
      \hskip1.5em plus0em minus0.5em \fi \nobreak
      \vrule height0.4em width0.5em depth0.25em\fi}
\def\Prob{\text{Prob}}
\def\prob{\text{Prob}}
\def\for{\text{for}}
\def\For{\text{For}}

\def\prox{{\bf Proximate ~Latent Points~}}

\def\bM{{\bf M}}
\def\bA{{\bf A}}
\def\bP{{\bf P}}
\def\bW{{\bf W}}
\def\bB{{\bf B}}
\def\Var{\mbox{Var}}

\def\Null{\mbox{Null}}
\def\Span{\mbox{Span}}
\def\bQ{{\bf Q}}

\def\bR{{\bf R}}
\def\bF{{\bf F}}
\def\bG{{\bf G}}
\def\bR{{\bf R}}
\def\bS{{\bf S}}

\usepackage{amsmath,amssymb}

\begin{document}

\title{Finding a latent $k-$simplex in $O^{*}\left(k \cdot \mbox{{\tt nnz(data)}}\right)$ time via Subset Smoothing}
\author{Chiranjib Bhattacharyya\\
Department of Computer Science and Automation\\
Indian Institute of Science\\
Bangalore, 560012 \\
India \\
\and
Ravindran Kannan\\
Microsoft Research India\\
\\
Bangalore, 560001\\
India
}
\date{}
\maketitle
\fancyfoot[R]{\scriptsize{Copyright \textcopyright\ 2020 by SIAM\\
Unauthorized reproduction of this article is prohibited}}

\pagenumbering{arabic}
\setcounter{page}{1}

\begin{abstract}
In this paper we show that the learning problem for a large class of Latent variable models,
such as Mixed Membership Stochastic Block Models, Topic Models, and Adversarial
Clustering can be posed as the problem of {\it Learning a Latent Simplex (LLS)}:
\emph{find a latent $k-$ vertex simplex, $K$ in ${\bf R}^d$, given $n$ data
points, each obtained by perturbing a latent point in $K$}. Our main contribution
is an efficient algorithm for LLS
under deterministic assumptions which naturally hold for the models considered here.

We first observe that for a suitable $r\leq n$, $K$ is close to 
a data-determined polytope $K'$ 
(the \emph{subset smoothed polytope}) which is the convex hull of the ${n\choose r}$
points, each obtained by averaging an $r$ subset of data points. 
Our algorithm is simply
stated: it optimizes
$k$ carefully chosen linear functions over $K'$ to find the $k$ vertices of the latent simplex.
The proof of correctness is more involved, drawing on existing and new tools from Numerical Analysis. 
Our running time is $O^*(k\mbox{ nnz})$ (This is the time taken by one itertion of the $k-$means algorithm.)
This is better than all previous algorithms for the special cases when data is sparse, as is the norm for Topic Modeling and
MMBM.
Some consequences of our algorithm are:
\begin{itemize} \item Mixed Membership Models and Topic Models: We give
the first quasi-input-sparsity time algorithm for $k\in O^*(1)$. 
\item Adversarial Clustering: In $k-$means, an adversary is allowed to move many data points
from each cluster towards the convex hull of other cluster centers.
Our algorithm still estimates cluster centers well.
\end{itemize}
\end{abstract}

\section{Introduction}
Understanding the underlying generative process of observed data is an important goal of Unsupervised learning.  The setup is assumed to be stochastic where
each \emph{observation} is
generated from a \emph{probability distribution} parameterized by a model.
Discovering such models from observations is a challenging problem,
often intractable in the general.
$k-$means Clustering of data is a simple and important special case. It is often used on data which is
assumed to be generated by a mixture model like Gaussian Mixtures, where,
 all observations are  generated from a \emph{fixed} convex combination of density functions.
The $k$-means problem, despite its simplicity, poses challenges and
continues to attract considerable research attention.
Mixed membership models~\cite{A14} are interesting generalizations of
Mixture models, where instead of a \emph{fixed} convex combination, each
observation arises from a different convex combination, often determined stochastically.
Special cases of Mixed membership models include Topic Models~\cite{survey},
and Mixed Membership Stochastic Block(MMSB) models~\cite{A08}
which have gained significant attention for their ability to model real-life data.

The success of discovered models, though approximate, in explaining real-world data has spurred interest in deriving algorithms with proven upper bounds on error and time which can recover
the true model from a \emph{finite} sample of observations.
This line of work has seen growing interest in algorithms community ~\cite{AGM,anima, Hop17,tsvd}
. We note that these papers make domain-specific assumptions under which the algorithms are analyzed. 

Mixture Models, Topic Models, and MMSB are all instances of \emph{Latent Variable models}.
Our aim in this paper is to arrive at a simply stated, general algorithm which is applicable to large class of
Latent Variable models and which in polynomial time can recover the true model
if certain general assumptions are satisfied.
We take a geometric perspective on latent variable models and
argue that such a perspective can serve as a unifying view yielding a single algorithm which is competitive with the state of the art and indeed
better for sparse data.

\section{Latent Variable models and Latent k-simplex problem}
This section  reviews three well-known Latent variable models: Topic models (LDA), MMSB and Clustering.
The purpose of the review is to bring out the fact that all of them can be viewed abstractly as special
cases of a geometric formulation which we call the Learning a Latent Simplex (LLS) problem. 

Given (highly)pertubed points
from a $k-$simplex in ${\bf R}^d$, learn the $k$ vertices of $K$. 

In total generliaty, this problem is intractable. 
A second purpose of the review of LDA and MMSB is to distill out the domain-specific model assumptions they make into
general determinsitic geometric assumptions on LLS. With these deterministic assumptions in place, our main contribution of the paper is to devise a fast algorithm to solve LLS.

\subsection{Topic models and LDA}

Topic Models attempt to capture underlying themes of a document by
topics, which are probability distributions over all words in the vocabulary.
Given a corpus of documents, each document is represented by
relative word frequencies. Assuming that the corpus is generated from an \emph{ad-mixture} of these $k-$distributions or $k-$topics, the core goal of Topic models is to construct the underlying $k$ topics.
These can also be viewed geometrically as  $k$ latent vectors $M_{\cdot,1},M_{\cdot,2},\ldots ,
M_{\cdot, k}\in {\bf R}^d$ ($d$ is the size of the vocabulary)
where, $M_{i,\ell}$ is the expected frequency of
word $i$ in topic $\ell$. Each $M_{\cdot,\ell}$ has non-negative
entries summing to 1.
In Latent Dirichlet Allocation (LDA) (\cite{lda}), an important example of Topic models, the
data generation is stochastic.
 A document consisting of $m$ words is generated by the following two stage process:
\begin{itemize}
\item The topic distribution of  Document $j$ is decided by the topic weights, $W_{\ell,j}, \ell=1,2,\ldots ,k$ picked from a Dirchlet distribution on
the unit simplex $\{ x\in {\bf R}^k: x_\ell\geq 0;\sum_{\ell=1}^kx_\ell=1\}$.
The topic of the document
$j$ is set to $P_{\cdot,j}=\sum_{\ell=1}^kM_{\cdot, \ell}W_{\ell,j}$.
$P_j$ are latent points.
\item The $m$ words
of document $j$ are generated in i.i.d. trials from the multinomial distribution with $P_{\cdot,j}$ as the
probability vector.
The data point $A_{\cdot,j}$ is given by:
\end{itemize}
$$A_{i,j}=\frac{1}{m}\sum_{t=1}^{m} X_{ijt}, \quad X_{ijt} \sim Bernoulli(P_{ij})$$
The random-variate, $X_{ijt} = 1$ if $i$ th word was chosen in the $t$ th draw
 while generating the $j$ th document and $0$ otherwise.
In other words $A_{ij}$ is the relative frequency
of $i$th word in the $j$th document. As a consequence,
\begin{equation}\label{eq:stat}
E(A_{ij}) = P_{ij}\; ;\; Var(A_{ij}| P_{ij}) = \frac{1}{m} P_{ij} (1 - P_{ij})
\end{equation}
The data generation process of a Topic model, such as LDA, can be also viewed
from a geometric perspective.  For each document, $j$, the observed data, $A_{\cdot,j}$, is generated from $P_{\cdot,j}$, a point in the simplex, $K$, whose vertices are defined by the $k-$ Topic vectors.

If priors are sepcified then
recovering $\bM$ in Topic Modeling, such as LDA, can be viewed as a classical parameter estimation problem, where, one is given samples
(here $A_{\cdot,j}$) drawn from a multi-variate probability distribution and the problem is to estimate the
parameters of the distribution.
The learning problem in LDA is usually addressed by two classes of algorithms-
Variational and MCMC based, neither of which is known to be polynomial time bounded.
Recently,
polynomial time algorithms have been
developed for Topic Modeling under assumptions on word frequencies,
topic weights and Numerical Analysis properties (like condition number) of the matrix $\bM$\cite{AGM,anima,tsvd}.
While the algorithms are provably polynomial time bounded, the assumptions are 
domain-specific and the running time
is not good as the algorithm we present here. Also, it is to be noted that the
algorithm to be presented here  is completely
different in approach from the ones in the literature.

Topic modeling can be posed geometrically as the problem of learning the latent k-simplex, namely, the
convex hull of the topic vectors.
To formalize 
the problem and devcise an algorithm it will be useful to understand
some properties of the data generated from such simplices,
which we do presently. 
For concretenss, we focus on LDA, a widely used Topic model.

{\bf Data Outside $K$}: 
The convex hull of the topic vectors $M_{\cdot,1},M_{\cdot,2},\ldots ,M_{\cdot,k}$ we assume is a
simplex (namely, they span a $k$ dimensional space). It is denoted $K$. $K$ is in ${\bf R}^d$ and so
$d\geq k$; indeed, generally, $d>>k$.

We point out here that data points $A_{\cdot,j}$ can lie outside $K$.
Indeed, even in the case when $k=1$, whence, $K=\{ M_{\cdot, 1}\}$,
$A_{\cdot,j}$ will lie outside $K$ since:
in the usual parameter setting, $m,k\in O^*(1)$, and,
$n$ goes to infinity and the tail of $M_{\cdot,\ell}$, namely,
$I=\{ i: M_{i,\ell}\leq 1/(2m)\}$ has $\sum_{i\in I}M_{i,\ell}\in \Omega(1)$.
So, for most $j$, there is at least one $i\in I$ with $A_{ij}\geq 1/m$ which implies
$A_{\cdot,j}\not= M_{\cdot,1}$. While for $k=1$, there are simpler examples, we have
chosen this illustration because it is easy to extend the argument for general $k$.

 Not only does data often lie outside $K$, we argue below that in fact it lies
 significantly outside. This property of data being outside
$K$ distinguishes our problem from the extensive literature in Theoretical Computer Science (see e.g. \cite{AGR13})
on learning polytopes given (often uniform random) sample inside the polytope.
A simple calculation, which easily follows from \eqref{eq:stat},
shows that
\begin{align*}
&E(|A_{\cdot,j}-P_{\cdot,j}|^2)=\frac{1}{m}(1-\sum_{i=1}^dP_{ij}^2) \\
&\geq \frac{1}{m}(1-\Max_\ell\sum_{i=1}^dM_{i\ell}^2),
\end{align*}
where the last inequality follows by noting that $P_{\cdot,j}$ is a convex combination of columns of $\bM$ and $x^2$ is a convex function of $x$.
\begin{equation}\label{1235}
\sum_{i=1}^dM_{i\ell}^2\leq \Max_{i,\ell}M_{i,\ell}=\gamma, \mbox{  say }.
\end{equation}
Now, $\gamma$ the maximum frequency of a single word in a topic,
which is usually at most a small fraction.
So, 
individual $|A_{\cdot,j}-P_{\cdot,j}|$ which we refer to as the ``perturbation'' of point $j$,
is $\Omega(1/\sqrt m)$, which is $\Omega^*(1)$ in the usual parameter
ranges. But note that a side of $K$, namely $|M_{\cdot,\ell}-M_{\cdot,\ell'}|\leq 1$,
so perturbations can be the same order as sides of $K$. To summarize in words: most/all 
data points can lie
$\Omega(\mbox{ side length of }K)$ outside of $K$.

{\bf Subset Averages} While individual $|A_{\cdot,j}-P_{\cdot,j}|$ may be high, intuitively, the average of
$A_{\cdot, j}$ over a large subset $R$ of $[n]$ should be close to the average of $P_{\cdot,j}$ over $R$ by law of
large numbers. Indeed, we will prove later  
an upper bound on the spectral norm of the matrix of pereturbations $\bA-\bP$
(see Lemma \ref{LDA-TM} gives a precise statement) by using the stochastic independence of words in documents and applying Random Matrix Theory.
This upper bound immediately implies that simultaneously for every $R\subseteq [n]$,
we have a good upper bound on $|A_{\cdot,R}-P_{\cdot,R}|$ as we will see in Lemma (\ref{averages}).
This leads us to our starting point for an algorithmic solution, namely, a technique we will call 
{\it Subset Smoothing}. Subset Smoothing is the obserbvation that if we take the convex hull $K'$
of the ${n\choose \delta n}$  averages of all $\delta n$-sized  subsets of the $n$ data points, then,
$K'$ provisdes a good approximation to $K$ under our assumptions.  
(Theorem ~\ref{K-K-prime} states precisely how close $K,K'$ are.)
We next describe another property of LDA which is essential for subset smoothing.

{\bf Proximity}
Recall that LDA posits a Dirichlet prior on picking topic weights $W_{\ell,j}$.
The Dirichlet prior $\phi(\cdot)$ on the unit simplex
is given by
$$\phi(W_{\cdot,j})\propto \prod_{\ell=1}^k W_{\ell ,j}^{\beta-1},$$
where, $\beta$ is a parameter, often set of a small positive value like $1/k$. Since $\beta <1$, this
density puts appreciable mass near the corners. Indeed, one can show (see Lemma (\ref{dirichlet})) that
\begin{equation}\label{eq:prox}
\forall \ell, \forall \zeta \quad \Prob\left( W_{\ell,j}\geq 1-\zeta \right)\geq \frac{1}{3k}\zeta^2.
\end{equation}
So, a good fraction of the $P_{\cdot,j}$ are near each corner of $K$. This indeed helps the learning
algorithm, since, for example, if all $P_{\cdot,j}$ lay in a proper subset of $K$, it is difficult to
learn $K$ in general. Quantitatively, we will see that this leads a lower bound on the fraction of
$j$ with $P_{\cdot,j}\approx M_{\cdot,\ell}$.

{\bf Input Sparsity}  Each data point $A_{\cdot,j}$ has at most
$m$ non-zero entries and as we stated earlier, typcally $m<<d,n$ and
so the data is sparse. It is important to design algorithms which exploit this sparsity.

\subsection{Mixed Membership Stochastic Block(MMSB) Models}\label{section:MMBM}

Formulated in \cite{A08}, this is a model of a random graph where edge $(j_1,j_2)$
is in the graph iff person $j_1$ knows person $j_2$.  There are $k$ communities; there is a $k\times k$
latent matrix $\bB$, where,
$B_{\ell_1,\ell_2}$ is the probability that a person in community $\ell_1$ knows
a person in community $\ell_2$. An underlying stochastic process is posited, again consisting of
two components: For $j=1,2,\ldots ,n$, person $j$ picks a $k$ vector $W_{\cdot,j}$ of community
membership weights (non-negative reals summing to 1) according to a Dirichlet probbaility density.
This is akin to the prior on picking $\bW$ in LDA.
Then, the edges are picked independently. For edge $(j_1,j_2)$, person $j_1$ picks a community $\ell$
from the multinomial with probabilities given by $W_{\cdot,j_1}$ and person $j_2$ picks a community
$\ell'$ according to $W_{\cdot,j_2}$. Then, edge $(j_1.j_2)$ is included
in $G$ with probability $B_{\ell,\ell'}$. So, it is easy to see that the probaility
$P_{j_1,j_2}$ of edge $(j_1,j_2)$ being in $G$ is given by $P_{j_1,j_2}=\sum_{\ell,\ell'}W_{\ell,j_1}B_{\ell,\ell'}W_{\ell',j_2}$
which in matrix notation reads:
$$\bP=\bW^T\bB\bW.$$

$\bW$ being a stochastic matrix, it cannot be recovered from $G$, but we can aim to recover $\bB$.
But now $\bP$ depends quadratically on $\bW$ and recovering $\bB$ directly does not seem easy. Indeed, the
only provable polynomial time algorithms known to date for this problem use tensor methods
or Semi-definite programming
and require
assumptions; further the running time is a high polynomial (see \cite{anima14,Hop17}).
But we can pose the problem of recovery of the $k$ underlying communities differently.
Instead of aiming to get $\bB$, we wish to pin down a vector for each community.
First we pick at random a subset $V_1$ of $d$ people and lets call the remaining set of
people $V_2$. For convenience, assume $|V_2|=n$.
We will represent community $\ell$
by a $d$ vector, where the $d$ ``features'' are the probabilities that
a person in $V_2$ in community $\ell$ knows each of the $d$ people in $V_1$.
With a sufficiently large $d$, it is intuitively clear that the community vectors
describe the communities well. Letting the columns of a 
$k\times d$ matrix $\bW^{(1)}$ and the columns of a $k\times n$ matrix $\bW^{(2)}$
denote the fractional membership weights of
members in $V_1$ and $V_2$ respectively, the probability matrix for the bipartite graph on $(V_1,V_2)$
is given by
\begin{equation}\label{eq:mmsb}
\bP=\underbrace{(\bW^{(1)})^T\bB}_{\bM}\bW^{(2)}.
\end{equation}
This reduces the Model Estimation problem here to our geometric problem: Given $\bA$, the adjacency matrix
of the bipartite graph, estimate $\bM$.
Note that the random variables in $\bW^{(1)},\bW^{(2)}$ are independent.

\cite{A08} assumes that each column of $\bW^{(2)}$ is picked from the Dirchlet distribution. An usual setting
of the concentration parameter of this Dirichlet is $1/k$ and we will use this value.
The proof that Proximate Data Assumption is satisfied is exactly on the same
lines as for Topic Models above. The proof that
spectral norm of the perturbation matrix $\bA-\bP$
 is small
again draws on Random Matrix Theory (\cite{vers}), but, is slightly simpler.
\noindent
MMSB also shares the four properties discussed in the context of LDA.

 {\bf Data Outside $K$} We illustatrte in a simple case. Suppose $k=1$ and $\bM$ has
all entries equal to $p$. The graph of who knows whom is generally sparse. This means
 $p\in o(1)$. $A_{\cdot,j}$ now will consist of $pd$ 1's (in expectation), so we will have
$|A_{\cdot,j}-M_{\cdot,1}|\approx \sqrt{pd}$, whereas $|M_{\cdot,\ell}|=p\sqrt d$. Since
$p\in o(1)$
here, in fact, $|M_{\cdot,1}|$ is $o(1)$ times perturbation, so indeed data is far away from $K$.
This example can be generalized to higher $k$.
More generally, if block sizes are each $\Omega(d)$ and we have graph sparsity, the same
phenomenon happens.

{\bf Subset Averages} Under the stochastic model, we can again use Random Matrix Theory
(\cite{vers}) to derive an upper bound on the spectral norm of $\bA-\bP$
similar to LDA. [The proof is different
because edges are now mutually independent.] Then as before, we can use Lemma ~\ref{averages}
 to show that
for all $R$, the averages of data points and latent points in $R$ are close.

 {\bf Proximity} Since the Drichlet density is used as a prior, we have the same argument as in LDA and one can prove a result similar \eqref{eq:prox}.

{\bf Input Sparsity} The graph of who knows whom is typically sparse.

\subsection{Adversarial Clustering}\label{sec:advinf}
Traditional Clustering problems arising from mixture models can be stated as:
Given $n$ data points $A_{\cdot,1},A_{\cdot,2},\ldots ,A_{\cdot,n}\in {\bf R}^d$ which can be partitioned
into $k$ distinct clusters $C_1,C_2,\ldots ,C_k$, find the means $M_{\cdot,1},M_{\cdot,2},\ldots ,M_{\cdot,k}$
of $C_1,C_2,\ldots ,C_k$. While there are many results
for mixture models showing that under stochastic assumptions, the $M_{\cdot,\ell}$
can be estimated, more relevant to our discussion are the results of \cite{KK10} and \cite{AS12},
which show that under a deterministic assumption, the clustering problem can be solved. In more detail,
letting $P_{\cdot,j}$ denote the mean of the cluster $A_{\cdot,j}$ belongs to (so, $P_{\cdot,j}\in\{ M_{\cdot,1},M_{\cdot,2},\ldots,
M_{\cdot,k}\}$) and defining $\sigma,\delta$ as follows
($\sigma$ denotes the maximum over directions of the square root of the mean-squared perturbation in the direction, and $\delta$
is a lower bound on the weight of a cluster):
\begin{equation}\label{eq:sing}
\sigma=\Max_{v:|v|=1}\sqrt{\frac{1}{n}|v^T(\bA-\bP)|^2}=\frac{1}{\sqrt n}||\bA-\bP||\; ;\;  \frac{1}{n}\Min_\ell|C_\ell|\geq\delta,
\end{equation}
\cite{KK10} and \cite{AS12} show that:

If $|M_{\cdot,\ell}-M_{\cdot,\ell'}|\geq ck\frac{\sigma}{\sqrt\delta}\forall \ell\not=\ell',$
the $M_{\cdot, \ell}$ can be
found within error $O(\sqrt k\sigma/\sqrt\delta)$.

Note that $\delta$ may go to zero as $n\rightarrow\infty$.
We observe (see Lemma ~\ref{lem:planted-clique} for a formal statement) that
if the error can be improved to $o(\sigma/\sqrt\delta)$ just in the case when $k=2$,
then, we can find $o(\sqrt n)$ size planted cliques in a random graph $G(n,1/2)$ settling a
major open problem. So, at the present state of knowledge, an error of $O(\sigma/\sqrt\delta)$
(times factors of $k$) is the best dependence of the error on $\sigma,\delta$
we can aim for and our algorithm will achieve this for the LLS problem.

\subsubsection{ Adversarial Noise}\label{sec:advnoise} 
We now allow an adversary to choose
for each $\ell\in [k]$, a subset $S_\ell$ of $C_\ell$ of cardinality $\delta n$ and
to add noise
$\Delta_j$
to each data point $A_{\cdot,j}$, where, the $\Delta_j$ satisfy the following conditions:
\begin{itemize}
\item For all $j$, $P_{\cdot,j}+\Delta_j\in \text{Convex Hull of } (M_{\cdot,1},M_{\cdot,2},\ldots ,M_{\cdot,k})$ and
\item For $j\in \cup_{\ell=1}^k S_\ell $, $|\Delta_j|\leq 4\sigma/\sqrt\delta$.
\end{itemize}
In words, each data point $A_{\cdot,j}$ is moved an arbitrary amount towards the convex hull of the means of the clusters it does not
belong to (which intuitively makes the learning problem more difficult),
but, for $\delta n$ points in each cluster, the move is by distance at most $O(\sigma/\sqrt\delta)$
. Note that the $C_\ell,S_\ell, P_{\cdot, j},\Delta_j$, the original
$A_{\cdot,j}$ are all
latent; only the adversarial-noise-added $A_{\cdot,j}$ are observed and the problem is to find the $M_{\cdot,\ell}$ approximately.
For convenience, we pretend that the same noise $\Delta_j$ has been added to the latent points
$P_{\cdot,j}$ as well, so $\bA-\bP$ remains invariant. Also for ease of notation, we
denote the noise-added data points, by $A_{\cdot,j}$ and the noise-added latent points by $P_{\cdot,j}$ from now on.

Now, it is easy to see that (the new) $P_{\cdot,j}$ satisfy the following:
\begin{align}
&\forall j, P_{\cdot,j}\in \text{Convex Hull of } (M_{\cdot,1},M_{\cdot,2},\ldots ,M_{\cdot,k})\label{Pjconditions-1}\\
&\forall \ell\in [k], \exists S_\ell, |S_\ell|=\delta n: \forall j\in S_\ell, |P_{\cdot,j}-M_{\cdot,\ell}|\leq 4\sigma/\sqrt\delta\label{Pjconditions-2}
\end{align}
Also, it is clear that any set of $P_{\cdot,j}$ satisfying these two conditions qualify as
a set of latent points for our Adversarial clustering problem.

{\bf Notation:}
$n,d,k$ are reserved for number of data points, number of dimensions of the space and the number of vertics
of $K$ respectively. Also, we reserve $i,i',i_1,i_2$ to index elements of $[d]$, $j,j',j_1,j_2$ to index $[n]$ and
$\ell,\ell',\ell_1,\ell_2$ to index $[k]$. $\bA,\bM,\bP$ are reserved for the roles described above.
$A_{\cdot,j}$ denotes the $j$ th column of matrix $\bA$ and so too for other matrices.
For a vector valued random variable $X\in {\bf R}^d$, $\Var(X)$ denotes the covariance matrix of $X$.
For a matrix $\bB$,
$s_1(\bB),s_2(\bB),\ldots $ are the singular values arranged in non-increasing order.
$||{\bB}||=\Max_{x:|x|=1}|x^T\bB |=s_1(\bB)$ is the spectral norm.
CH denotes convex hull of what follows. CH of a matrix is the convex hull of its columns.

\subsection{Latent k-simplex is an unifying model}\label{sec:lks}
From the review of existing models one can conclude that indeed learning many latent
variable models can be posed as the Latent k-simplex problem.
More precisely, the learning problem
can be understood as  aiming to recover the $k$ vertices
$M_{\cdot, 1},M_{\cdot,2},\ldots ,M_{\cdot,k}$
of a latent simplex $K$
in ${\bf R}^d$
given data generated from $K$ under a stochastic or a deterministic model.

In situations with hypothesized stochastic processes, the following assumptions
are made on the
data generation process (of generating $A_{\cdot,j}$ given $P_{\cdot,j}$):
\begin{align*}
&A_{\cdot,j}, j=1,2,\ldots ,n\mbox{ mutually independent }\; |\; \bP\\
&A_{\cdot,j} \mbox{ are drawan according to a specific prob distribution satisfying }\\
&E(A_{\cdot, j}\;|\; P_{\cdot, j})=P_{\cdot,j}\mbox{  and }\\
&\mbox{ An upper bound on } ||\Var(A_{\cdot,j}\; |\; P_{\cdot,j})||.
\end{align*}
Under these conditions, Random Matrix Theory \cite{vers}
can be used to prove that $||\bA-\bP||$ is bounded:
(See for example proof of Lemma \ref{LDA-TM} for a precise statement)
\begin{align}\label{895}
\frac{1}{\sqrt n}|||\bA-\bP||=\sigma&\leq O\left(||\Var(A_{\cdot,j}\; |\; P_{\cdot,j})||^{1/2}\right).
\end{align}
Also in both LDA and MMSB models, a Dirichlet prior is assumed on the $P_{\cdot,j}$.
More precisely, since $P_{\cdot,j}$ are convex combinations of the vertices of $K$,
there is a (latent) $k\times n$ matrix $\bW$ with non-negative entries and  column sums
equal to 1 such that
$$\bP=\bM\bW.$$
The Dirichlet density
(under the usual setting of parameters) is maximum at the corners of the simplex
and attaches at least a positive fraction of weight to the region close to each corner.
Thus, one has with high probability:
\begin{equation}\label{896}
\forall \ell\in [k]\; ,\; \prob\left( |P_{\cdot,j}-M_{\cdot,\ell}|\leq\varepsilon_1\right)\geq \delta_1,
\end{equation}
for suitable $\varepsilon_1,\delta_1$.

Here, we do not assume  a stochastic model of data generation, neither do we posit
any prior on $P_{\cdot,j}$.
Instead, we make deterministic assumptions.
To impose a bound on $\sigma$, analogous to (\ref{895}), we impose
an assumption we call  
{\bf Spectrally bounded Perturbations}.
Furthermore, to characterize the property that there is concentration of observed data near the extreme points, we make the assumption we call  \prox.
We will show later that in the LDA model as well as usual MMSB model, our
deterministic assumptions are satisfied with high probability.

The deterministic assumptions
 have another advantage: In Clustering and other
situations like NMF, there is geenrally no stochastic model, but we can still apply our
results.
An upper bound on spectral norm similar to what we use
here was used for the special cases of (non-adversarial) clustering and pure mixture models
(which are a special case of the ad-mixture models we have here)
in \cite{KK10}.
Also in the clustering context, the Proximity assumption is similar to the assumption
often made of a lower bound on cluster weights.

The main contribution of this paper is to show that Learning a Latent Simplex (LLS) 
problem can be solved using three deterministic assumptions: the two described above plus
an assumption of {\bf Well Separatedness}, a standard assumption in mixture models, that the vertices of $K$ are well-separated.

\section{Subset Smoothing}\label{section:smooth}

We now describe the starting point of our solution method, which is a technique we call ``Subset Smoothing''. It shows that
the simplex $K$ which we are trying to learn is well-approximated by a data-determined polytope $K'$
which is the convex hull of the ${n\choose \delta n}$ points, each of which is the average of $A_{\cdot,j}$
over $j$ in a subset of size $\delta n$. While the description of $K'$ is exponential-sized, it is easy to
see that $K'$ admits a polynomial-time (indeed linear time) optimization oracle and this will be our starting point.
First, to see why averages of $\delta n$ subsets of data help, note that our assumptions
will not place any upper on individual
perturbations $|A_{\cdot,j}-P_{\cdot,j}|$; indeed they are typically very large in Latent Variable Models as we discussed above. However, we will place an upper bound on $||\bA-\bP||$, the spectral norm of the perturbation matrix or
equivalently on $\sigma=||\bA-\bP||/\sqrt n$. [Such an upper bound on $||\bA-\bP||$ is usually available in
stochastic latent variable models via Random Matrix Theory, as we will see later in the paper for LDA and MMBM.]
For any subset
$R$ of $[n]$, denote by $A_{\cdot, R} $ the average of $A_{\cdot,j}, j \in R$, and so for $\bP$, namely
$$A_{\cdot, R}=\frac{1}{|R|}\sum_{j\in R}A_{\cdot,j}\; ;\; P_{\cdot, R}=\frac{1}{|R|}\sum_{j\in R}P_{\cdot,j}.$$
It is easy to see that an upper bound on $\sigma$ 
(defined in (\ref{eq:sing}))
implies an upper bound on $\Max_{R:|R|=\delta n}|A_{\cdot,R}-P_{\cdot,R}|$:

\begin{lemma}\label{averages}
For all $S\subseteq [n],$
$|A_{\cdot, S}-P_{\cdot, S}|\leq \frac{\sigma\sqrt n}{\sqrt {|S|}}.$
\end{lemma}
\begin{proof}
This just follows from the fact that $|A_{\cdot, S}-P_{\cdot, S}|=\frac{1}{|S|}|(\bA-\bP){\bf 1}_S|\leq \frac{1}{|S|}||\bA-\bP||\; |{\bf 1}_S|$ and
$|{\bf 1}_S|=\sqrt{|S|}$.
\hfill\end{proof}
Now, we can prove that the data-determined polytope $K'=CH(A_{\cdot, S}:|S|=\delta n)$ is close to the simplex
$K=CH(\bM)$ which we seek to find. Closeness of two sets $K_1,K_2$ is measured in Hausdorff metric $D(K_1,K_2)$ which we
define here.
For sets $K_1,K_2$, define:
$$\text{Dist}(K_1,K_2)=\sup_{x\in K_1}\inf_{y\in K_2} |x-y|.$$
Note Dist$(K_1,K_2)$ may not equal Dist$(K_2,K_1)$.
If $K_1$ is a single point $x$, we write $\text{Dist}(x,K_2)$.
$\text{Dist}(x,K_2)$ is a convex function of $x$. So, we have:

\begin{claim}\label{792}
If $x$ varies over
polytope $K_1$, then the maximum of $\text{Dist}(x,K_2)$ is attained at a vertex of $K_1$.
\end{claim}

Hausdorff diatsnce $D$ is defined by
$D(K_1,K_2)=\Max(\text{Dist}(K_1,K_2),\text{Dist}(K_2,K_1))$.

\begin{theorem}\label{K-K-prime}
Let $K'=CH(A_{\cdot,R}:|R|=\delta n)$. We have
$D(K,K')\leq 5\sigma/\sqrt\delta.$
\end{theorem}
\begin{proof}
We first prove that Dist$(K,K')\leq 5\sigma/\sqrt\delta$ for which, by Claim (\ref{792}),
it suffices to show that
$\forall \ell \in [k], \exists S\subseteq [n], |S|=\delta n: |M_{\cdot,\ell}-A_{\cdot,S}|\leq \frac{5\sigma}{\sqrt\delta}$.
Take $S=S_\ell$ (for the $S_\ell$ defined in (\ref{Pjconditions-2})). Since for each $j\in S_\ell$, we have $|P_{\cdot,j}-M_{\cdot,\ell}|\leq 4\sigma/\sqrt \delta$,
we have by convexity of $|\cdot |$ that $|P_{\cdot, S_\ell}-M_{\cdot,\ell}|\leq 4\sigma/\sqrt\delta$. By Lemma (\ref{averages}),
it follows that $|P_{\cdot,S_\ell}-A_{\cdot,S_\ell}|\leq\sigma/\sqrt\delta$. Adding these two and using the triangle
inequality, we get $|M_{\cdot,\ell}-A_{\cdot,S_\ell}|\leq \frac{5\sigma}{\sqrt\delta}$ as claimed.

To prove that Dist$(K',K)\leq 5\sigma/\sqrt\delta$,
again by Claim (\ref{792}), it suffices to show for any $S\subseteq [n]$, $|S|=\delta n$, Dist$(A_{\cdot,S},K)\leq\sigma\sqrt\delta$.
Note that $P_{\cdot,S}\in K$ (since it is the avergae of $\delta n$ points in $K$) and
also $|A_{\cdot,S}-P_{\cdot,S}|\leq\sigma/\sqrt\delta$ by Lemma (\ref{averages}) and so Dist$(A_{\cdot,S},K)\leq \sigma/\sqrt\delta$.
\end{proof}

\begin{lemma}\label{oracle}
Given any $u\in {\bf R}^d$, $\Max_{x\in K'} (u\cdot x)$ can be found in linear time
(in $\bA$).
\end{lemma}
\begin{proof} One computes $u\cdot A_{\cdot,j}, j=1,2,\ldots ,n$ by doing a matrix-vector product in time $O(\text{nnz}(\bA))$
and takes the avrage of the $\delta n$
highest values.
\end{proof}

The above immediately suggests the question: Can we just optimize $k$ linear functions over $K'$ and hope
that each optimal solution gives us an approximation to a new vertex of $K$? It is easy to see that if
we choose the $k$ linear functions at random, the answer is not necessarily. However, this idea does work
with a careful choice of linear functions. We will now state the choices which lead to our algorithm.

\section{Statement of Algorithm}

Our algorithm will choose (carefully) $k$ linear functions. We will show that optimizing each of these
will give us an approximation to a new vertex of $K$, thus at the end, we will have all $k$ vertices.
The algorithm can be stated in a simple self-contained way and we do so presently. We will prove correctness
under our assumptions after formalizing the assumptions.
However, the proof is not nearly as simple as the algorithm statement
and will occupy the rest of the paper. Of the steps in the algorithm, the truncated SVD step at the start is costly and does not meet our
time bound. We will later replace it by the classical subspace power iteration method which does.

\begin{algorithm}
\caption{An algorithm for finding latent k-polytope from data matrix ${\bf A}$}
\textbf{Input}: $\bA, k,\delta$  ~{\tt $\bA$~is a $(d \times n)$ matrix}\\

\begin{algorithmic}

\State Let $V$ be the vector space spanned by the top $k$ left singular vectors of $\bA$.

 \ForAll {$r=0,1,2,\ldots ,k-1$}

    \State Pick $u$ at random from the $k-r$ dimensional sub-space
$U=V\cap \Null(A_{\cdot, R_1},A_{\cdot, R_2},\ldots ,A_{\cdot,R_r})$.

    \State $R_{r+1}\leftarrow \arg\max_{S:|S|=\delta n}|u\cdot A_{\cdot,S}|$

\EndFor

\State {\bf Return:} $\{ A_{\cdot,R_1}, A_{\cdot,R_2}\ldots ,A_{\cdot,R_k}\}$ as approximation to
$\{ M_{\cdot,1},M_{\cdot,2},\ldots ,M_{\cdot,k}\}$.

\end{algorithmic}
\end{algorithm}

\section{Learning a Latent k-simplex (LLS) problem and Main results-Informal statements}
Before we state our results we informally describe the main results of the
paper.

Recall that the Latent k-simplex problem:
Given data points $A_{\cdot,j}, j=1,2,\ldots ,n\in {\bf R}^d$, obtained by perturbing latent points
$P_{\cdot, j}, j=1,2,\ldots ,n$ respectively, from a latent $k-$simplex $K$, learn the $k$ vertices.
Our main result is that there is a quasi-input sparsity time algorithm which
could solve this problem under certain assumptions.

\paragraph{Assumptions:} We will informally introduce the assumptions to
explain our results.
\begin{itemize}

\item {\bf Well-Separatedness}(Informal Statement)
Each of the $M_{\cdot,l}$ for  $\ell=1,2,\ldots ,k$,
has a substantial component orthogonal to the space spanned by the other $M_{\cdot,\ell'}$. This makes the vectors well separated.

\item\prox(Informal Statement)
$\forall \ell\in [k]$, there are at least
$\delta n$ $j$ 's with $P_{\cdot,j}$ close to (at distance at most $4\sigma/\sqrt\delta$) from $M_{\cdot,\ell}$.
$\delta$ is in $(0,1)$ and can depend on $n,d$, in particular going to zero as $n,d\rightarrow\infty$.
Note that in the case of $k-$means Clustering, all data have $P_{\cdot,j}=$ a vertex of $K$; there, $\delta$
is the minimum fraction of data points in any cluster.

\item{\bf Spectrally bounded perturbations}(Informal Statement) We assume
$$\frac{\sigma}{\sqrt\delta}\leq \frac{\Min_\ell|M_{\cdot,\ell}|}{\text{poly}(k)}.$$
It is clear that if the perturbations are unbounded then it is impossible to recover the true polytope. We have already
discussed above why the upper bound on $\sigma$ is reasonable.

\end{itemize}

Now we can state the main problem and result.

\paragraph{Learning a Latent Simplex (LLS) problem}(Informal Statement)
\emph{
Given $n$ data points $A_{\cdot,j}, j=1,2,\ldots ,n\in {\bf R}^d$ such that there is an unknown $k-$simplex $K$ and
unknown points $P_{\cdot,j} \in K, j=1,2,\ldots ,n$,
find approximations to vertices of $K$ within error poly$(k)\sigma/\sqrt\delta$. [I.e., find
$\widetilde M_{\cdot, \ell}, \ell =1,2,\ldots ,k$ such that there is some permutation of indices with
$|M_{\cdot,\ell}-\widetilde M_{\cdot,\ell}|\leq \text{poly}(k)\sigma/\sqrt\delta \; \for\; \ell=1,2,\ldots ,k.$
]}

The main result can be informally stated as follows:
\begin{theorem}\label{main-theorem-informal} (Informal Statement)
If observations, generated through {\bf Spectrally  bounded Perturbation}
of latent points generated from a polytope with {\bf well-separated} vertices, satisfy the
\prox assumption, then the main problem can be solved in
time $O^*(k\times\mbox{nnz}(\bA)+k^2d)$.
\footnote{$O^*$ hides logarithmic factors in $n,k,d$ as well as factors in $\delta,\alpha$.}

\end{theorem}
We will develop an algorithm which approximately recovers the vertices of the Latent k-simplex and achieves the run-time complexity mentioned in the theorem.

\paragraph{Literature related to learning simplices}
In Theoretical Computer Science and Machine Learning,
there is substantial literature on Learning Convex Sets \cite{V10,KDS08},
 intersection of half spaces \cite{KS07,V10,V10a},
Parallelopipeds\cite{FJK} and simplices \cite{AGR13} However,
this literature does not address our problem since it assumes we are given  data
points \emph{which are all in the convex set}, whereas, in our settings, as we saw, they
are often (far) outside.

There are a number of algorithms in Unsupervised Learning as mentioned above. But algorithms with proven time bounds
have two issues which prevent their use in our problem: (a)
all of them depend
on context-specific technical assumptions and (b) They have worse time bounds.
The quasi-input-sparsity complexity is a very attractive feature of our
algorithm and the generality of the problem makes it applicable to wide range of latent variable models such as MMSB, Topic Models, and Adversarial Clustering.
It is to be noted that our method also gives an immediate quasi-input-sparsity algorithm for $k-$ means clustering for $k\in O^*(1)$.
 We are not aware of any such result in the literature (see Section ~\ref{sec:inp-sps}).

\section{Assumptions, Subset Smoothing, and Main results}
In this section we formally describe the key assumptions, and our main
results. We derive an algorithm which uses subset smoothing and show that
it runs in quasi-input sparsity time.

\subsection{Assumptions}
As informally introduced before, the three main assumptions
are necessary for the development of the algorithm.
They crucially depend on the following parameters.
\begin{itemize}
\item The Well-separatedness of the model depends on $\alpha$, which  is a real number in $(0,1)$. We assume under Well-separatredness that each $M_{\cdot,\ell}$ has
a substantial component, namely, $\alpha \Max_{\ell'}|M_{\cdot,\ell'}|$,
orthogonal to the span of the other $M_{\cdot,\ell'}$.
Note that of course, the component of $M_{\cdot,\ell}$ orthogonal to span of other $M_{\cdot,\ell'}$ cannot be greater than
$|M_{\cdot,\ell}$, so implicit in this assumption, we require all $|M_{\cdot,\ell'}$ to be within $\alpha$ factor of each other.
$\alpha$ is an arbitrary model-determined parameter, so it can depend on $k$, but not on $n,d$.
Higher the value of $\alpha$ the more well-separated the model is.
\item The parameter $\delta \in (0,\frac{1}{k})$, quantifies the fraction of
 data close to each of the vertices.  \prox assumption requires that for each $\ell$, $\delta$ fraction of the $n$
latent points lie close to each vertex of $K$.


\end{itemize}

\begin{assum}

{\bf Well-Separatedness} We assume that there is an $\alpha\in (0,1)$ such that $\bM$ matrix obeys the following
\begin{equation}\label{condition-number-100}
\forall \ell\in [k], \left| \proj\left( M_{\cdot,\ell}\; ,\; \Null\left( \bM\setminus M_{\cdot,\ell}\right)\right)\right|\geq
\alpha \Max_{\ell'} |M_{\cdot,\ell'}|.
\end{equation}
\end{assum}
Well-Separatedness is an assumption purely on the model $\bM$ and is not determined by the data.
\begin{assum}
\prox: The model satisfies $\prox$ assumption if
\begin{align}\label{extreme-docs}
& \For \; \ell\in [k], \exists S_\ell\subseteq [n], |S_\ell|=\delta n,\text{ with }
|M_{\cdot,\ell}-P_{\cdot,j}|\le \frac{4 \sigma}{\sqrt\delta}\forall j\in S_\ell.
\end{align}
\end{assum}
\begin{remark}
Note that $\delta$ is always at most $1/k$ and  is allowed to be smaller,
it is allowed to go to zero as $n\rightarrow\infty$
\end{remark}
\begin{assum}
{\bf Spectrally Bounded Perturbations}
The following relationship will be assumed,
\begin{equation}\label{spectral-bound}
\frac{\sigma}{\sqrt\delta}\leq \frac{\alpha^3\Min_\ell|M_{\cdot,\ell}|}{4500k^9}.
\end{equation}
\end{assum}
This assumption depends on the observed data and somewhat weakly on the model.
The reader may note that
in pure mixture models, like Gaussian Mixture Models, a standard assumption is:
Component means are separated by $\Omega^*(1)$ standard deviations. (\ref{spectral-bound})
is somewhat similar, but not the same: while we have $\Min_\ell |M_{\cdot,\ell}|$ on the right
hand side, the usual separation assumptions would have $\Min_{\ell\not= \ell'}|M_{\cdot,\ell}-M_{\cdot,\ell'}|$.
While these are similar, they are incomparable.


\begin{remark}\label{separatedness-remark}
It is important that we only have poly$(k)$ factors and no factor dependent on $n,d$
in the denominator of the right hand side. Since
$n,d$ are larger than $k$, a dependence on $n,d$ would have been too strong a requirement and
generally not met in applications. Of course our dependence on $k$ could use improvement.
The factor of $\sigma/\sqrt\delta$ seems to be necessary at the current state of
knowledge, otherwise one can
solve the planted clique problem in $o(\sqrt{n})$ regime in polynotmal time.
A formal statement is provided in Lemma ~\ref{lem:planted-clique}
\end{remark}

\subsection{A Quasi Input Sparsity time Algorithm for finding extreme points of K}\label{alg-section}

\noindent
In this subsection we present an algorithn based on the Assumptions and Subset smoothing described earlier.
The algorithm is the same as described in Section ~\ref{section:smooth}, but with the first step of computing the
exact truncated SVD replaced by the classical subspace power iteration which meets the time bounds.
The algorithm proceeds in $k$ stages (recall $k$ is the number of vertices of $K$),
in each stage maximizing $|u\cdot x|$ over $x\in K'$ for a carefully chosen $u$.
The maximization can be solved by just finding all the $u\cdot A_{\cdot ,j}$ and taking the largest
(or smallest) $\delta n$ of them. Unlike the algorithm, the proof of correctness is not so simple.
Among the tools it uses are the $sin-\Theta$ theorem in Numerical Analysis, an
extension which we prove, and the properties of random projections.
A brief introduction to this and some basic properties may be found again in Section \ref{subspace}.

\begin{algorithm}
\caption{{\bf LKS}: An algorithm for finding latent k-simplex from data matrix ${\bf A}$}
\textbf{Input}: $\bA$ \Comment{$\bA$~is a $(d \times n)$ matrix}\\
\textbf{Input}:  $k$  \Comment{ k is the number of vertices}\\
\textbf{Input}: $\delta$  \Comment{ $\delta$  between $0$ and $\frac{1}{k}$}\\
\textbf{Input}: $t$  \Comment{ $t= c \log d$ ~where~$c$ is a constant }\\

\begin{algorithmic}

\State ${\bf Q_t} =$ {\tt Subspace-Power}($\bA, t$)
\State Let $V=\Span({\bf Q_t})$

 \ForAll {$r=0,1,2,\ldots ,k-1$}

    \State Pick $u$ at random from the $k-r$ dimensional sub-space
$U=V\cap \Null(A_{\cdot, R_1},A_{\cdot, R_2},\ldots ,A_{\cdot,R_r})$.

    \State $R_{r+1}\leftarrow \arg\max_{S:|S|=\delta n}|u\cdot A_{\cdot,S}|$

\EndFor

\State {\bf Return:} $A_{\cdot,R_1}, A_{\cdot,R_2}\ldots ,A_{\cdot,R_k}$.

\end{algorithmic}
\end{algorithm}

\begin{theorem}\label{main-theorem}
Suppose we are given $k\geq 2$ and data $\bA$, satisfying the assumptions of 
{\bf Well-Separatedness}
(\ref{condition-number-100}),  \prox  (\ref{extreme-docs}),
and {\bf Spectrally Bounded Perturbations} (\ref{spectral-bound}).
Then, in time $O^*\left(k\left(\text{nnz}(\bA)+kd\right)\right)$ time, the Algorithm {\bf LKS} finds subsets $R_1,R_2,\ldots,R_k,
$ of cardinality $\delta n$ each such that after a permutation of columns of $\bM$, we have with
probability at least $1-(c/\sqrt k)$:
$$|A_{\cdot, R_\ell}-M_{\cdot,\ell}|\leq \frac{150k^{4}}{\alpha} \frac{\sigma}{\sqrt\delta}\; \for\; \ell=1,2,\ldots ,k.$$
\end{theorem}
%
%
%
%

We next state the main result which directly implies theorem (\ref{main-theorem}). The hypothesis of the result below
is that we have already found $r\leq k-1$ columns of $\bM$ approximately, in the sense that we have found
$r$ subsets $R_1,R_2,\ldots ,R_r\subseteq [n], |R_t|=\delta n$ so that there are $r$ distinct columns
$\{ \ell_1,\ell_2,\ldots ,\ell_r\}$ of $\bM$ with $M_{\cdot ,\ell_t}\approx A_{\cdot, R_t}$ for $t=1,2,\ldots ,r$.
The theorem gives a method for finding a $R_{r+1}, |R_{r+1}|=\delta n$ with $A_{\cdot ,R_{r+1}}\approx M_{\ell}$ for some
$\ell\notin \{ \ell_1,\ell_2,\ldots ,\ell_r\}$.

Theorem (\ref{main-theorem}) follows by applying Theorem (\ref{main-technical-theorem}) $k$ times.

\begin{theorem}\label{main-technical-theorem}
Suppose we are given data $\bA$ and $k\geq 2$ satisfying the assumptions of {\bf Well-Separatedness} 
(\ref{condition-number-100}), \prox  (\ref{extreme-docs})
and {\bf Spectrally Bounded Perturbations} (\ref{spectral-bound}).
Let $r\leq k-1$. Suppose $R_1,R_2,\ldots ,R_r\subseteq [n]$, each of cardinality $\delta n$ have been found and
are such
that there exist $r$ distinct elements $\ell_1,\ell_2,\ldots ,\ell_r\in [k]$, with:\footnote{We do not know $\bM$ or
$\ell_1,\ell_2,\ldots ,\ell_r$, only their existence is known.}
\begin{equation}\label{A-M}
|A_{\cdot, R_t}- M_{\cdot, \ell_t}|\leq  \frac{150k^4}{\alpha}\frac{\sigma}{\sqrt\delta}
\mbox{ for } t=1,2,\ldots ,r.
\end{equation}
Let $V$ be any $k-$ dimensional subspace of ${\bf R}^d$ with
$\sin\Theta(V,\Span(v_1,v_2,\ldots ,v_k))\leq \sigma/\sqrt\delta $ (where, $v_1,v_2,\ldots ,v_k$ are the
top $k$ left singular values of $\bA$).
Suppose $u$ is a random unit length vector in the $k-r$ dimensional sub-space $U$ given by:
$$U=V\cap \mbox{Null} (A_{\cdot, R_1},A_{\cdot, R_2},\ldots ,A_{\cdot, R_r})$$ and suppose
$$S=\arg\max_{T\subseteq [n], |T|=\delta n} |u\cdot A_{\cdot, T}|.$$
Then, with probability at least $1-(c/k^{3/2})$,
$$\exists \ell \notin \{\ell_1,\ell_2,\ldots ,\ell_r\} \mbox{  such that  } |M_{\cdot, \ell}-A_{\cdot, S}|\leq  \frac{150k^4}{\alpha}\frac{\sigma}{\sqrt\delta}
 .$$
Further this can be carried out in time $O^*(\text{nnz}(\bA)+dk)$ time.
\end{theorem}

\subsubsection{Quasi-Input-Sparsity Based Complexity}\label{sec:inp-sps}

While there has been much progress on
devising algorithms for sparse matrices and indeed for Low-Rank Approximation (LRA)
nearly optimal dependence of $O^*(\text{nnz})$ on input sparsity is known
\cite{CW13}, there has not been such progress on standard $k-$means Clustering (for which several
algorithms first do LRA). This is in spite
of the fact that there are many instances where the data for Clustering problems is very sparse.
For example, Graph Clustering is a well-studied area and many graphs tend to be sparse.
Our complexity dependence on nnz is $k \times {\tt  nnz}$, and hence we refer it as \emph{quasi-input-sparsity} time complexity
when $k\in O^*(1)$. We do not have a proof of a corresponding lower bound. But recently,in \cite{MW17}, it was argued  that for kernal LRA, $k$ nnz is possibly optimal unless matrix multiplication
can be improved.
We leave as
an open problem the optimality of our {\tt nnz} dependence. We are unaware of  an algorithm for any of the special cases, considered in this paper, has
a better complexity than ours.

\section{Latent Variable Models as special cases}\label{special-cases}

In this section we discuss three latent variable models LDA, MMSB and Adversarial Clustering
and prove that they are special cases of our general geometric problem.

\subsection{LDA as a special case of Latent k-simplex problem:}
In the LDA setup we will consider that the prior is $\mbox{Dir}(\frac{1}{k},k)$ on the unit simplex.
The following arguments apply to Dir$(\beta,k)$ for any $\beta\leq 1/k$, but we do the case when $\beta=1/k$
here. We also assume, what we
call ``lumpyness'' of $M_{\cdot,\ell}$ which
intuitively says that for any $\ell$, $M_{i,\ell}, i=1,2,\ldots ,d$
should not all be small, or in other words, the vector $M_{\cdot,\ell}$
should be ``lumpy''.
We assume that
$$|M_{\cdot,\ell}|\in\Omega(1).$$
This assumption is
consistent with existing literature.
It is common practice in Topic Modelling to
assume that in every topic there are a \emph{few} words which have very high probability.
A topic with weights distributed among all (or many) words is not informative about the theme.
Furthermore, if word frequencies satisfy power law, it is indeed easy to see this assumption holds.
It is to be noted,  a weaker assumption than power law, namely, that the $O(1)$ highest frequency words together have $\Omega(1)$ frequency, is also enough to imply
our assumption that $|M_{\cdot ,\ell}|\in\Omega(1)$, as is easy to see by summing. If even this last assumption is violated,
it means say that a large number of  high frequency words still do not describe the topic well which wouldn't make for
a reasonably interpretable topic model.
\begin{lemma}\label{LDA-TM}
Suppose $\bA,\bP$ are as above. Assume that $|M_{\cdot,\ell}|\in\Omega(1)$ for all $\ell$. Suppose
$W_{\cdot,j}, j=1,2,\ldots ,n$ are i.i.d. distributed according to Dir$(k,1/k)$ and
assume $m,n$ are at least a sufficiently large polynomial function of $\frac{k}{\alpha}$.
 Let $\delta =\frac{c\sigma}{\sqrt k}$.
Then, (\ref{extreme-docs}) and (\ref{spectral-bound}) are satisfied with high probability.
\end{lemma}
\begin{proof} See Section \ref{sec:top-proof}. \end{proof}
\begin{remark}
The only assumption for which we have made no assertion is Well-Separatedness (\ref{condition-number-100}).
This is because, there is no generally assumed prior for generating $\bM$. If one were to assume a Dirichlet
prior for $\bM$ with sufficiently low concentration parameter or assume a power law frequency distribution
of words, one can show (\ref{condition-number-100}) holds.
\end{remark}

\begin{theorem}\label{LDA-TM-2}
Suppose $\bA,\bP$ are as in Lemma ~\ref{LDA-TM}. 
Assume that $|M_{\cdot,\ell}|\in\Omega(1)$ for all $\ell$
and $n\geq m$.
Suppose
$W_{\cdot,j}, j=1,2,\ldots ,n$ are i.i.d. distributed according to Dir$(k,1/k)$ and
assume $m,n$ are at least a sufficiently large polynomial function of $k/(\alpha) $.
Let $\delta =\frac{c\sigma}{\sqrt k}$. Also assume the Well-Separatedness
assumption (\ref{condition-number-100}) is satisfied. Then, our algorithm with high probability finds
approximations $\widetilde M_{\cdot,1},\widetilde M_{\cdot,2},\ldots ,\widetilde M_{\cdot,k}$ to the
topic vectors so that (after a permutation of indices),
$$|\widetilde M_{\cdot,\ell}-M_{\cdot,\ell}|\leq \frac{ck^{4.5}}{ \alpha m^{1/4}}\; \for\; \ell=1,2,\ldots ,k.$$
\end{theorem}


\subsection{MMSB Models}

The assertions and proofs here are similar to LDA. The difference is in the proof of an upper bound
on $\sigma$ (Spectrally Bounded Perturbations), since, here, all the edges of the graph
(entries of $\bA$) are mutually
independent, but there is no effective absolute (probability 1) bound on perturbations.

$\bA,\bM,\bP,{\bf W^{(2)}}$ have the meanings discussed in Section ~\ref{section:MMBM},
see equation \eqref{eq:mmsb}. We introduce one more symbol here: we let $\nu$ denote the
maximum expected degree of any node in the bipartite graph, namely,
$$\nu=\Max(\Max_i\sum_jP_{ij},\Max_j\sum_iP_{ij}).$$
Instead of the ``lumpyness'' assumption that $|M_{\cdot,\ell}|\in \Omega(1)$ we made in Topic Modeling,,
we make the assumption here that $|M_{\cdot,\ell}|\geq\nu^{1/8}$. The reader can verify that
this won't be satisfied if $\nu$ is small and $P_{ij}$ are spread out, but will be if $\nu$ is
at least $d^\gamma$ for a small $\gamma$.

\begin{lemma}\label{MMSB-1}
Suppose $\bA,\bP$ are as above. Assume $|M_{\cdot,\ell}|\geq\nu^{1/8}$ for all $\ell$ and $n\geq d$.
Also suppose $n/d$ is a sufficiently high polynomial in $\frac{k}{\alpha}$..
Suppose $W_{\cdot,j}^{(2)}$ are i.i.d. distributed according to $Dir(k,1/k)$. Let $\delta =c\sigma/\sqrt k$.
Then, (\ref{extreme-docs}) and (\ref{spectral-bound}) are satisfied with high probability.
\end{lemma}
\begin{theorem}\label{MMSB-2}
Suppose $\bA,\bP$ are as in Lemma ~\ref{MMSB-1}. 
Assume $|M_{\cdot,\ell}|\geq\nu^{1/8}$ for all $\ell$ and $n\geq d$.
Also suppose $n/d$ is a sufficiently high polynomial in $k/\alpha$, where, $\varepsilon\in [0,1]$.
Suppose $W_{\cdot,j}^{(2)}$ are i.i.d. distributed according to $Dir(k,1/k)$. Let $\delta =c\sigma/\sqrt k$.
Suppose in addition, (\ref{condition-number-100}) holds. Then the algorithm finds $\widetilde M_{\cdot,\ell},
\ell=1,2,\ldots ,k$ such that after a permutation, we have (whp)
$$|\widetilde M_{\cdot,\ell}-M_{\cdot,\ell}|\leq \frac{ck^{4.5}(\nu d)^{1/8}}{\alpha n^{1/4}}.$$

\end{theorem}

\begin{remark} By making $n$ sufficiently larger than $\nu, d$, we can make the error small.
\end{remark}

\subsection{Adversarial Clustering}

There is a latent ground-truth $k-$Clustering $C_1,C_2,\ldots ,C_k$
with cluster centers $M_{\cdot,1},M_{\cdot,2},\ldots ,M_{\cdot, k}$.
There are $n$ latent points $P_{\cdot,j}, j=1,2,\ldots ,n$ with $P_{\cdot,j}=M_{\cdot,\ell}$ for all $j\in C_\ell$.
The following assumptrions are satisfied:
\begin{enumerate}
\item \label{clustering-I} Well-Separatedness condition (\ref{condition-number-100}).
\item \label{clustering-II} Each cluster has at least $\delta n$ data points.
\item \label{clustering-III} Spectrally bounded perturbations (\ref{spectral-bound}).
\end{enumerate}

We then allow an
adversary to introduce noise as in Section ~\ref{sec:advnoise}.

\begin{theorem}\label{adv-clustering}
Given as input data points after adversarial
noise as above has been introduced, our algorithm finds (the original)
$M_{\cdot,\ell}$ to within error as in
Theorem (\ref{main-theorem}).
\end{theorem}

%
\begin{remark} \label{925} We want to point out that
the traditional $k-$means clustering does not solve the Adversarial Clustering problem.
A simple example in one dimension is: The original $K$ is $[-1,+1]$ and $n/2$ points are in each
cluster with a small $\sigma$. We then move $n(0.5-\delta)$ from the cluster centered at -1 each by
+0.5
and $n(0.5-\delta )$
points from cluster centered at +1 by -.5 each. It is easy to see that the best $2-$means clustering of the noisy
data is to locate two cluster centers, one near each of $\pm .5$, (depending on $\delta$), not near $\pm 1$.
\end{remark}

\section{Closeness of Subspaces and Subspace power iteration}\label{subspace}
In this section we will present the classical Subspace power iteration algorithm
which finds an approximation to the subspace spanned by the top $k$ (left) singular vectors
of $\bA$. It has a well-known elegant proof of convergence, which also we present here, since,
usual references often present more general (and also more complicated) proofs. Let
$$\mbox{SVD}(\bA)=\sum_{t=1}^ds_t(\bA) v_tu_t^T.$$

\subsection{Closeness of Subspaces}
First, we recall a measure of closeness
of sub-spaces.
Numerical Analysis has developed, namely, the notion of angles between sub-spaces, called Principal
angles. Here, we need only one of the principal angles which we define now.

For any two sub-spaces $F,F'$ of ${\bf R}^d$, define
$$\sin\Theta(F,F')=\Max_{u\in F}\Min_{v\in F'} \sin\theta (u,v)=\Max_{u\in F, |u|=1}\Min_{v\in F'}|u-v|.$$
$$  \cos\Theta(F,F')=\sqrt{ 1-\sin\Theta^2(F,F')}.$$
The following are known facts about $\sin\Theta$ function:
If $F,F'$ have the same dimension and the columns of
$\bF$ (respectively $\bF'$) form an orthonormal basis of $F$ (respectivel $F'$), then
\begin{eqnarray}\label{cos-F-F-prime}
\cos\Theta(F,F')& =s_{\text{Min}}(\bF^T\bF')\nonumber \\
\cos\Theta(F',F) & =\cos\Theta(F,F')\nonumber\\
\tan\Theta(F,F') &=||\bG^T\bF'(\bF^T\bF')^{-1}||,
\end{eqnarray}
where, the columns of matrix $\bG$ form a basis for $F^\perp$, and assuming the inverse of $\bF^T\bF'$ exists.

An important Theorem due to Wedin\cite{wedin72}, also known as the $\sin\Theta$ theorem, proves a bound on the $\sin\Theta$ between SVD-subspaces of a matrix and its perturbation:

\begin{theorem}\cite{wedin72}\label{wedin}
Suppose $\bR,\bS$ are any two $d\times n$ matrices. Let $m\leq \ell$ be any two positive integers. Let $S_m(\bR),S_\ell (\bS)$
denote respectively the subspaces spanned by the top $m$, respectively top $\ell$ left singular vectors of $\bR$, respectively,
$\bS$. Suppose $\gamma =s_m(\bR)-s_{\ell+1}(\bS) >0$. Then,
$$\sin\Theta (S_m(\bR),S_\ell (\bS))\leq \frac{||\bR-\bS||}{\gamma}.$$
\end{theorem}
Corollary below says the SVD subspace of $\bA$ is approximately equal to the Span of $\bP$.
\begin{corollary}\label{wedin-2}
Let $\bA,\bP$ be defined in Section ~\ref{sec:lks}.
$$\sin\Theta(\Span(v_1,v_2,\ldots ,v_k),\Span(\bP))\leq\frac{||\bA-\bP||}{s_k(\bA)}. $$
\end{corollary}
\begin{proof}
Apply Theorem (\ref{wedin}) with $\bR=\bA,\bS=\bP,m=\ell=k$
\end{proof}
We next find SVD subspace of $\bA$ by subspace power iteration.

\subsection{Subspace-Power Iteration}\label{power-method}

\begin{algorithm}[h]
\caption{{Subspace Power Method}}\label{alg:spm}
\begin{algorithmic}
\Function{{\tt Subspace-Power}}{$\bA$, $T$}
\State \textbf{Input}:$\bA$ and $T$ \Comment{T is the number of iterations}
\State \emph{Initialize:} $Q_0$ be a random $d\times k$ matrix with orthonormal columns.

\ForAll{$t=1,2,\ldots,  T$}
\State Set ${\bf Z_{t}}= \bA\bA^T {\bf Q_{t-1}}.$
\State Do Grahm-Schmidt on ${\bf Z_{t}}$ to obtain ${\bf Q_t}$
\EndFor

\Return ${\bf Q_{T}}$. $V=\Span({\bf Q_T})$.
\EndFunction
\end{algorithmic}
\end{algorithm}

Recall that $v_1,\ldots,v_d$ are left singular vectors of $\bA$ and
$\bQ_T$ is the $T$ iterate of the subspace power iteration.
\begin{theorem}\label{convergence-power}
{\bf Convergence of Subspace power iteration}
\begin{equation}\label{power-iteration}
\sin\Theta (\Span(v_1,v_2,\ldots ,v_k), \Span ({\bf Q_{c\ln d}}))   \leq \frac{\alpha^2}{1000k^9}.
\end{equation}
\end{theorem}
\begin{proof}
Let ${\bf F},{\bf G}$ be the $d\times k,d\times (d-k)$ matrices with columns $v_1,v_2,\ldots ,v_k$, respectively
$v_{k+1},v_{k+2},\ldots ,v_d$ and $F,G$ be respectively the subspaces spanned by their columns.
Note that from (\ref{cos-F-F-prime}), we have
$$\tan\Theta(F, Q_t)=||{\bf G}^T{\bf Q_t}\left( \bF^T{\bf Q_t}\right)^{-1}||.   $$
\begin{lemma}
$$\tan\Theta(F, Q_t)\leq \left( \frac{s_{k+1}(\bA)}{s_k(\bA)}\right)^{2t}||{\bf G}^T{\bf Q_0}\left( \bF^T{\bf Q_0}\right)^{-1}||.$$
\end{lemma}
\begin{proof}
${\bf Q_t}={\bf Z_t}{\bf R}$, where $\bR$ is the invertiblke matrix of Grahm-Scmidt orthogonalization, We write

Define $DS_{i:j} = \mbox{Diag}(s_{i}^{2}(\bA),s_{i+1}^{2}(\bA),\ldots s_j(\bA)^{2})$ for positive integers $i,j$ such that $i < j$.
\begin{align*}
&||{\bf G}^T{\bf Q_t}\left( \bF^T{\bf Q_t}\right)^{-1}||=||{\bf G}^T{\bf Z_t}\bR\bR^{-1}\left( \bF^T{\bf Z_t}\right)^{-1}||  \\
&=||DS_{k+1:d} {\bf G}^T{\bf Q_{t-1}}\left( F^T{\bf Q_{t-1}}\right)^{-1}
DS_{1:k}^{-1}||\\
&\leq \frac{s_{k+1}^2}{s_k^2}||{\bf G}^T{\bf Q_{t-1}}\left( \bF^T{\bf Q_{t-1}}\right)^{-1}||   ,
\end{align*}
since, $\bG^T{\bf Z_t}=\bG^T\bA\bA^T{\bf Q_{t-1}}=\bG^T\sum_{t=1}^ds_t^2(\bA)v_tv_t^T{\bf Q_{t-1}}=DS_{k+1:d}\bG^T{\bf Q_{t-1}}$
and similarly, $\bF^T{\bf Z_t}=DS_{1:k}\bF^T{\bf Q_{t-1}}$. This proves the Lemma by induction on $t$.
\hfill\end{proof}
We prove in Claim ~\ref{singular-P-M} that $s_k(\bP)\geq 4\sigma\sqrt n$. Also, $s_{k+1}(\bA)\leq s_{k+1}(\bP)+||\bA-\bP||
=\sigma\sqrt n$. Thus, we have
$s_{k+1}/s_k\leq 1/2$ and now applying Lemma above, Theoirem (\ref{convergence-power}) follows,
since for a random choice of ${\bf Q_0}$, we have $||{\bf G}^T{\bf Q_0}\left( \bF^T{\bf Q_0}\right)^{-1}||\leq \mbox{poly}(nd)$.
\hfill\end{proof}

Next, we apply Wedin's theorem to
prove Lemma (\ref{sin-theta}) below which
says
that any $k$ dimensional space $V$ with small $\sin\Theta$ distance to Span$(v_1,v_2,\ldots ,v_k)$ also has small
$\sin\Theta$ distance to Span$(\bM)$. We first need a technical Claim.

\begin{claim}\label{singular-P-M}
Recall that $s_t$ denotes the $t$ th singular value.
$$s_k(\bM)\geq\frac{1000k^{8.5}}{\alpha^2}\frac{\sigma}{\sqrt\delta}\; ;\; s_k(\bP)\geq \frac{995k^{8.5}\sqrt n}{\alpha^2}\sigma.$$
\end{claim}
\begin{proof}
$s_k(\bM)=\Min_{x:|x|=1} |Mx|$. For any $x, |x|=1$, there must be an $\ell$ with $|x_\ell|\geq 1/\sqrt k$.
Now, $|\bM x|\geq |proj(\bM x, \Null(\bM\setminus M_{\cdot,\ell}))|=|x_\ell | |proj(M_{\cdot,\ell}, \Null(\bM\setminus M_{\cdot,\ell}))|
\geq \alpha |M_{\cdot,\ell}|/\sqrt k\geq 1000k^9\sigma/(\alpha^2\sqrt\delta \sqrt k)$ by (\ref{condition-number-100}) and (\ref{spectral-bound})
proving the first assertion of the Claim.

Now, we prove the second.
Recall there are sets $S_1,S_2,\ldots ,S_k\subseteq [n]$ with $\forall j\in S_\ell, |P_{\cdot,j}-M_{\cdot,\ell}|\leq \frac{4\sigma}{\sqrt\delta}$.
We claim the $S_\ell$ are disjoint: if not, say, $j\in S_\ell\cap S_{\ell'}$. Then, $|P_{\cdot,j}-M_{\cdot,\ell}|, |P_{\cdot,j}-M_{\cdot,\ell'}|
\leq 4\sigma/\sqrt\delta$ implies $|M_{\cdot,\ell}-M_{\cdot,\ell'}|\leq 8\sigma/\sqrt\delta\leq \alpha^2\Min_\ell|M_{\cdot,\ell}|/100k^9$
by (\ref{spectral-bound}). But, by (\ref{condition-number-100}), $|M_{\cdot,\ell}-M_{\cdot,\ell'}|\geq |\proj(M_{\cdot,\ell},\Null(\bM\setminus M_{\cdot,\ell})|
\geq\alpha |M_{\cdot,\ell}|$ producing a contradiction.

Let $\bP'$ be a $d\times k\delta n$ sub-matrix of $\bP$ with its columns $j\in S_1\cup S_2\cup\ldots \cup S_k$ and let $\bM'$ be the
$d\times k\delta n$ matrix with $M'_{\cdot,j}=M_{\cdot,\ell}$ for all $j\in S_\ell, \ell=1,2,\ldots ,k$. We have $s_k(\bM')\geq \sqrt{\delta n}s_k(\bM)
\geq 1000k^{8.5}\sqrt n\sigma/\alpha^2$. Now, $||\bP'-\bM'||\leq \sqrt{k\delta n}4\sigma /\sqrt\delta$. Since
$s_k(\bP)\geq s_k(\bP')\geq s_k(\bM')-||\bP'-\bM'||$,
the second part of the claim follows.
\hfill\end{proof}

\begin{lemma}\label{sin-theta}
Let $v_1,v_2,\ldots ,v_k$ be the top $k$ left singular vectors of $\bA$.
Let $V$ be any $k-$ dimensional sub-space of ${\bf R}^d$ with
$$\sin\Theta(V,\Span(v_1,v_2,\ldots ,v_k))\leq\frac{\alpha^2}{1000k^9}.$$

For every unit length vector $x\in V$, there is a vector $y\in \Span(\bM)$
with $$|x-y|\leq \frac{\alpha^2}{500k^{8.5}}.$$
\end{lemma}
\begin{proof}
Since $\Span (\bP)\subseteq \Span(\bM)$, it suffices to prove the Lemma with $y\in$ Span$(\bP)$.
Corollary ~\ref{wedin-2} implies:
$$\mbox{ Sin }\Theta\left( \Span(v_1,v_2,\ldots v_k), \mbox{Span}(\bP)\right)
\leq $$
$$\frac{||\bA-\bP||}{s_k(\bA)}\leq \frac{||\bA-\bP||}{s_k(\bP)-||\bA-\bP||}\leq \frac{\sigma}{(995k^{8.5}\sigma/\alpha^2)-\sigma}\leq \frac{\alpha^2}{994k^{8.5}}.$$

Now, $\sin\Theta(V,\Span(v_1,v_2,\ldots ,v_k))\leq\alpha^2/1000k^9$ and,
$\sin\Theta (V,\Span(\bM))\leq \sin\Theta(V,\Span(v_1,v_2,\ldots ,v_k)) +\sin\Theta( \Span(v_1,v_2,\ldots ,v_k),\Span(\bM))$,
which together imply the Lemma. [We have used here the triangle inequality for $\sin\Theta$ which follows directly from the definition
of $\sin\Theta$.]
\hfill\qed
\end{proof}


\section{Technical Lemmas}
In this section we prove technical claims which are useful to support the main
claims of the paper. We begin by noting an useful connection between the choice of $\delta$ and the famous planted clique problem.
\begin{lemma}\label{lem:planted-clique}{\bf Planted Clique and Choice of $\delta$}
Suppose $K$ has just two vertices with one of them being the origin and the other equal to
${\underline 1}_Q$ for an unknown subset $Q$ of $[n]$ with $|Q|=q=\delta n$. Let $\bP,\bA,\sigma$ be as in our notation
and suppose a susbet $Q'$ of $[n]$
with  $|Q'|=q$ have each $P_{\cdot,j}$ equal to ${\underline 1}_Q$ and the rest $n-q$
of the $P_{\cdot,j}=$ the origin. Suppose $\bA$ is a random $\pm 1$ matrix with
$$\prob(A_{ij}=1)=\begin{cases} 1&\mbox{ if } i\in Q\; ;\; j\in Q'\\
          0.5&\mbox{ otherwise } .\end{cases}.$$
If ${\underline 1}_Q$ can be found within error at most
$\varepsilon \sigma/\sqrt\delta$, and $|Q|\geq 10\sqrt\varepsilon \sqrt n$, then, $Q$ can be found exactly.
\end{lemma}
\begin{proof}
$E(A_{ij})=P_{ij}$ and $E(A_{ij}-P_{ij})^2\leq 2$ as is easy to check. Al;so $A_{ij}$ are mutually independnet.
This by Ranodm Matrix Theory (\cite{vers}) implies that $\sigma\leq 4$. $\delta =q/n$. Now the
conclusion follows from the results of \cite{FK08}.\hfill\qed 
\end{proof}

\subsection{Topic Models obey \prox and bounded perturbation assumption}\label{sec:top-proof}
We present the proof of Lemma~\ref{LDA-TM} and Theorem~\ref{LDA-TM-2}. First, we prove that the Spectrally Bounded Perturbations hold.
\begin{proof}(of Lemma ~\ref{LDA-TM})
Note that	$|A_{\cdot,j}-P_{\cdot,j}|\leq ||A_{\cdot,j}-P_{\cdot,j}||_1\leq 2$.
Let
${\bf \Sigma_j}=E((A_{\cdot,j}-P_{\cdot,j})(A_{\cdot,j}-P_{\cdot,j})^T)$ be the covariance matrix of $A_{\cdot,j}$ and let
${\bf \Sigma}=\frac{1}{n}{\bf \Sigma_j}$.
From Theorem 5.44 of \cite{vers}, we get that with probability at least $1-\varepsilon$,
\begin{equation}\label{RMT-bound}
\sigma\leq \sqrt{||\Sigma||}+\frac{c}{\sqrt n},
\end{equation}
where, $c$ includes factors in $1/\varepsilon$.
The higher order term here is $\sqrt{||\Sigma||}$ which
the following lemma bounds.
\begin{lemma}\label{topic-model-norm}
With high probability,
$\sqrt{||{\bf\Sigma}||}\leq \frac{c}{\sqrt m} .$
\end{lemma}

\begin{proof}
Let $X_{ijt}=1$ or 0 according as the $t$ th word of document $j$ is the $i$ vocabulary word or not.
$$||{\bf \Sigma_j}||=\frac{1}{m}\Max_{|v|=1}E(\sum_{i=1}^d(v_i\cdot (X_{ijt}-P_{ij}))^2 $$\\
$$\leq \frac{1}{m}\max_{|v|=1}\left[ \Max_iP_{ij}-2\sum_{i_1\not= i_2}v_{i_1}v_{i_2}P_{i_1j}P_{i_2j} \right]$$
 using distribution of $X_{ijt}$.
$$ \leq \frac{1}{m}\Max_iP_{ij}+ \frac{1}{m}\Max_{|v|=1}\left(-(\sum_iv_iP_{ij})^2+\sum_iv_i^2P_{ij}^2\right)$$
$$ \leq \frac{2}{m}\Max_iP_{ij}\quad \implies ||{\bf\Sigma}||\leq 2/m.$$
\hfill
\end{proof}
In the hypothesis of  $\delta = c\sigma/\sqrt k$ in Theorem (\ref{LDA-TM}), the following inequality implies (\ref{spectral-bound}):

$$\sigma\leq \frac{c\alpha^6\varepsilon^4\mbox{Min}_\ell|M_{\cdot,\ell}|^2}{10^6k^{17}}.$$

This is in turn implied by the follwng:
\begin{equation}\label{spectral-bound-2}
\sqrt{||\Sigma||}+\frac{c}{\sqrt n}\leq \frac{c\alpha^6\varepsilon^4\mbox{Min}_\ell|M_{\cdot,\ell}|^2}{10^6k^{17}},
\end{equation}
which we now prove by showing that each of the two terms on the lhs is at most 1/2 the rhs.
Lemma (\ref{topic-model-norm}) plus the hypothesis that $m$
is a sufficiently large polynomial in $k$ and $|M_{\cdot,\ell}|\in\Omega(1)$ shows the desired upper bound on
$\sqrt{||\Sigma||}$. So, it only remains to bound the lower order term, namely, prove that
$\frac{c}{\sqrt n}\leq  \frac{c\alpha^6\varepsilon^4\mbox{Min}_\ell|M_{\cdot,\ell}|^2}{10^6k^{17}}.$
This follows by noting that $n$ is at least a sufficiently high polynomial
in $k$, and $|M_{\cdot,\ell}|\in\Omega(1)$ which proves (\ref{spectral-bound}).

Now, we turn to proving the
(\ref{extreme-docs}) assumption.
For this, we first need the following fact about the Dirichlet density.
\begin{lemma}\label{dirichlet}
If $x$ is a random $k-$ vector picked according to the Dir$(1/k,k)$ density on $\{ x:x_\ell\geq 0;\sum_\ell x_\ell=1\}$,
then for any $\zeta\in [0,1]$, we have
$$\prob\left( x_1\geq 1-\zeta\right)\geq \frac{\zeta^2}{3k}.$$
\end{lemma}
\begin{proof}
The marginal density $q(x_1)$ of the first coordinate of $x$ is easily seen to be
$$q(x_1)=cx_1^{(1/k)-1}(1-y)^{1-(1/k)},$$
where, the normalizing constant $c\geq 1/k$. For $y\in (0,1), y^{(1/k)-1}\geq 1$,
so $q(x_1)\geq \frac{1}{k}(1-x_1)^{1-(1/k)}$. Now integrating over $x_1\in [1-\zeta,1]$, we get the
lemma.
\hfill
\end{proof}
Thus, with $\delta = c\sigma/\sqrt k$ as assumed in Lemma~\ref{LDA-TM}, we get that the $S_\ell$ defined
in (\ref{extreme-docs}) satisfies $|S_\ell|\geq\delta n$, using H\"offding-Chernoff bounds.
This finishes the proof of Lemma~\ref{LDA-TM}.
\hfill \end{proof}

\begin{proof}(of Theorem~\ref{LDA-TM-2}) Note that its hypothesis also assumes (\ref{condition-number-100}).
So Theorem~\ref{main-theorem}
implies that the algorithm finds $\widetilde M_{\cdot,\ell},\ell\in [k]$ with
$$|M_{\cdot,\ell}-\widetilde M_{\cdot,\ell}|\leq \frac{k^{3.5}\sigma}{\alpha\varepsilon\sqrt\delta}
\leq \frac{ck^4\sqrt\sigma}{\alpha \varepsilon}\leq \frac{ck^4}{\alpha\varepsilon m^{1/4}},$$
the last using the upper bound on $\sigma$ of $(c/\sqrt m)+(c/\sqrt n)$ we proved and noting that
$n\geq m$.
\hfill
\end{proof}

\subsection{MMSB}
We sketch the proof of Lemma~\ref{MMSB-1} and Theorem~\ref{MMSB-2} omitting details, since the proof is somewhat similar to the proof in the LDA case. Let ${\bf D} $ denote a $d\times n$ matrix, where, $D_{ij}=\Var(A_{ij}\; |\; \bP)$. Clearly we have
$D_{ij}=P_{ij}(1-P_{ij})\leq P_{ij}$ and so $\sum_iD_{ij}\leq \nu$ for all $j$ and similary for row sums. Also, $\sum_{i,j}D_{ij}\leq \nu (\Min (d,n)=\nu d$. Latala's theorem \cite{Lat05} implies that with high probability,
$$||\bA-\bP||\leq c\Max ( \sqrt\nu, (\nu d)^{1/4})=c (d\nu)^{1/4}$$
which implies $ \sigma \leq \nu^{1/4}(d/n)^{1/4}.$
Since $\delta=c\sigma/\sqrt k$, to prove (\ref{spectral-bound}), it suffices to prove that
$$\sigma\leq \frac{\alpha^6\varepsilon^4 \Min_\ell|M_{\cdot,\ell}|^2}{10^6k^{17}},$$
which follows since the hypothesis of the Lemma says $n/d$  is a high polynomial in $k/\alpha\varepsilon$ and
$|M_{\cdot,\ell}|\geq \nu^{1/8}$.
This proves (\ref{spectral-bound}). The argument for \eqref{extreme-docs} is identical to the case of LDA.
This proves the Lemma. For the Theorem, we just have to show that upper bound the error guaranteed by Theorem \ref{main-theorem}.
satisfies the upper bound claimed here. This is straightforward using the above upper bound on $\sigma$ (and the fact that
$\delta=c\sigma/\sqrt k$).

\section{Proof of Correctness of the Algorithm}
In this section we prove the correctness of the algorithm described in Section (\ref{alg-section}) and establish the time complexity.

\subsection{Idea of the Proof}
The main contribution of the paper is the algorithm,
stated formally in Section (\ref{alg-section}) to solve the general problem
and the proof of correctness.
The algorithm itself is simple. It has $k$ stages; in each stage, it maximizes
a carefully chosen linear function $u\cdot x$
over $K'$; we prove that the optimum gives us
an approximation to one vertex of $K$.

\subsubsection{First Step}\label{first-step}
For the first step, we will pick a random unit vector $u$ in the $k$ dimensional SVD subspace of $\bA$.
This subspace is close to the sub-space spanned by $K$. In Stochastic models, the stochastic independence of the
data is used to show this (see for example \cite{VW02}). Here, we have not assumed any stochastic model.
Instead, we use a classical
theorem called the $\sin\Theta$ theorem \cite{wedin72} from Numerical Analysis. The $\sin\Theta$ theorem
helps us prove that the top singular subspace of dimension $k$ of $\bA$ is close to the span of $K$.
Now by our Well-Separatedness assumption, for $\ell\not=\ell'$, we will see that
$M_{\cdot,\ell}-M_{\cdot,\ell'}$
has length at least poly$(k)\sigma/\sqrt\delta$. For a random $u\in K$,
the $O(k^2)$ events that $|u\cdot (M_{\cdot,\ell}-M_{\cdot,\ell'})|\geq |M_{\cdot,\ell}-M_{\cdot,\ell'}|/\text{poly}(k)$
all happen simultaneously by Johnson-Lindenstrauss (and union bound.)
This is proved rigorously in Lemma (\ref{u-good}). [Note that had we picked $u$ uniformly at random from all of ${\bf R}^d$,
we can only assert that
$|u\cdot (M_{\cdot,\ell}-M_{\cdot,\ell'})|\geq |M_{\cdot,\ell}-M_{\cdot,\ell'}|/\sqrt d\text{poly}(k)$;
the $\sqrt d$ factor is not good enough to solve our problem.]

So, if we optimize $u\cdot x$ over $K$, the optimal $x$ is a vertex $M_{\cdot,\ell}$ with
$u\cdot M_{\cdot,\ell}$ substantially greater than any other $u\cdot M_{\cdot,\ell'}$. But we can only
optimize over $K'$. Since we make \prox assumption, there is a $\delta$ fraction of $j$ with their $P_{\cdot,j}\approx M_{\cdot,\ell}$,
(an assumption formally stated in (\ref{extreme-docs})), and so there is a $R\subseteq [n], |R|=\delta n$ with $P_{\cdot, R}\approx M_{\cdot,\ell}$
and so $A_{\cdot,R}\approx M_{\cdot,\ell}$ implying $u\cdot A_{\cdot,R}\approx u\cdot M_{\cdot,\ell}$.
Our optimization over $K'$ may yield some other subset $R_1$ with $u\cdot A_{\cdot,R_1}\widetilde > u\cdot M_{\cdot, \ell}$.
We need to show that whenever any subset
\footnote{For reals $a,b$, we say $a\widetilde >b$ if $a>b-$(a small number).}
$R_1$ has $u\cdot A_{\cdot,R_1}\widetilde > u\cdot M_{\cdot,\ell}$, it is also
close to $M_{\cdot,\ell}$ in distance. This is in Lemma (\ref{S-good}). Intuitively, the proof has two parts: We show that
any $R_1$ with high $u\cdot A_{\cdot,R_1}$ also has high $u\cdot P_{\cdot,R_1}$. But $P_{\cdot, R_1}$ is a convex
combination of columns of $\bM$. If the convex combination puts non-trivial weight on vertices other than $M_{\cdot,\ell}$,
its dot product with $u$ would go down since
we argued above that
$u\cdot M_{\cdot,\ell'}<<u\cdot M_{\cdot,\ell}$ for $\ell'\not=\ell$. So,
in the convex combination of the vertices of $K$ yielding $P_{\cdot,R}$ most of the weight must be on $M_{\cdot,\ell}$
and we use this to show that $|P_{\cdot,R}-M_{\cdot,\ell}|$ is small and so also $|A_{\cdot,R}-M_{\cdot,\ell}|$ is small. Thus
the optimal $A_{\cdot,R}$ serves as a good approximation to one vertex of $K$.

\subsubsection{General Step}\label{general-step}
In a general step of the algorithm, we have already found $r\leq k$ subsets $R_1,R_2,\ldots ,R_r\subseteq [n], |R_\ell|=\delta n$
with $A_{\cdot,R_\ell}\approx M_{\cdot, \ell}$ for $\ell=1,2,\ldots ,r$ (after some permutation of the indices of $M_{\cdot, \ell}$).
We have to ensure that the next stage gets an approximation
to a \emph{new} vertex of $K$. This is non-trivial. Random choice of the next $u$, even constrained to be
in the $k-$SVD subspace of $\bA$ need not work: the probability that it works depends on angles of the
simplex $K$ and these can be exponentially (in $k$) small. We overcome this problem with a new, but
simple idea: Choose a random vector $u$ from  a $k-r$ dimensional subspace $W$ obtained by intersecting
the SVD $k-$ subspace of $\bA$
with the NULL SPACE of $A_{\cdot,R_1}, A_{\cdot, R_2},\ldots ,A_{\cdot,R_r}$
and find $\Max_{x\in K'}|u\cdot x|$.
[The absolute value is necessary.]

If (i) $u$ had been in $W'=\Span(\bM)\cap \Null(M_{\cdot,1},M_{\cdot, 2},\ldots ,M_{\cdot, r})$,
and (ii) we optimized over $K$ instead of $K'$,
the proof would be easy from well-separateness. Neither is true. Overcoming (i) requires a
new Lemma (\ref{VcapNull-Mnull}) which proves that $W,W'$ are close in $\sin\Theta$ distance.
[The $\sin\Theta$ distance between $W,W'$ is $\Max_{x\in W}\Min_{y\in W'}\sin(\angle (x,y))$.]
 This is a technically
involved piece. This in a
way extends the Sin-Theta theorem in that it proves that if subspaces $W_1,W_2$ are close in Sin-Theta distance and
matrices $\widetilde M,\widetilde A$ are close, then $W_1\cap \Null (\widetilde M)$ and $W_2\cap \Null (\widetilde A)$
are also close under some conditions (that $\widetilde A,\widetilde M$ are far from singular) which do hold here.

We overcome (ii) in a similar way to what we did for the first step.
But, now this is more complicated by the fact that the $M_{\cdot,\ell}, M_{\cdot,\ell'}$ and $u$ have components along $M_{\cdot,1}, M_{\cdot, 2},
\ldots ,M_{\cdot,r}$ as well.

\subsection{Proof of the main theorem}
We are now ready to prove Theorem~\ref{main-technical-theorem}
\begin{proof}(of Theorem \ref{main-technical-theorem}):
Let
\begin{align*}
&\widetilde \bM=\left( M_{\cdot,\ell_1}\; |\; M_{\cdot,\ell_2}\; |\; \ldots \; |\; M_{\cdot,\ell_r}\right)\\
&\widetilde \bA =  \left( A_{\cdot,R_1}\; |\; A_{\cdot,R_2}\; |\; \ldots \; |\; A_{\cdot,R_r}\right)\\
\end{align*}

We next derive an extension of the classical $\sin\Theta$ theorem \cite{wedin72} which could be of general interest
but is crucial to the proof of the main theorem.
Intuitively, it  says that if we take close-by $k$ dim spaces and intersect them with null spaces of
close-by matrices, with not-too-small singular values,
then the resulting intersections are also close (close in $\sin\Theta$ distance). The reader may consult
Section (\ref{general-step}) for the role played by this Lemma in the proof of correctness.

\begin{lemma}\label{VcapNull-Mnull}
Under the hypothesis of Theorem \ref{main-technical-theorem}, we have:
\begin{align}
\sin\Theta\left( U \; ,\; \mbox{Span}(\bM)\cap \mbox{Null}(\widetilde \bM)\right)& \leq \frac{\alpha}{100k^4}\nonumber \\
\sin\Theta\left(\mbox{Span}(\bM)\cap \mbox{Null}(\widetilde \bM)\; ,\; U\right)&\leq \frac{\alpha}{100k^4}.\label{sin-theta-M-U},
\end{align}
\end{lemma}
\begin{proof}
For the first assertion, take $x\in U, |x|=1$. We wish to produce a $z\in \mbox{Span}(\bM)\cap \mbox{Null}(\widetilde \bM)$
with $|x-z|\leq \alpha/100k^4$.
Since $x\in V$, by Lemma \ref{sin-theta},
\begin{equation}\label{x-y}
\exists y\in \mbox{Span }(\bM): |x-y|\leq \frac{\alpha^2}{500k^{8.5}} .
\end{equation}

Let,
$z=y-\widetilde\bM(\widetilde \bM^T \widetilde\bM)^{-1} \widetilde \bM^T y$
be the component of $y$ in Null$(\widetilde \bM)$.
[Note: $\widetilde \bM^T\widetilde \bM$ is invertible since $s_r(\widetilde \bM)=\Min_{w:|w|=1}|\widetilde \bM w|\geq \Min_{x:|x|=1}|\bM x|=s_k(\bM)$ and
Claim (\ref{singular-P-M}).]
Since $y\in \Span(\bM)$, $z\in\Span(\bM)$ too.

\begin{align}\label{norm-100}
&|| \widetilde\bM(\widetilde \bM^T \widetilde\bM)^{-1} \widetilde \bM^T || \leq 1,
\end{align}
since it is a projection operator.
We have
\begin{align*}
&|y-z|=|\widetilde\bM(\widetilde \bM^T \widetilde\bM)^{-1} \widetilde \bM^T y| \\
&\leq |\widetilde\bM(\widetilde \bM^T \widetilde\bM)^{-1} \widetilde \bM^T (y-x)|+|\widetilde\bM(\widetilde \bM^T \widetilde\bM)^{-1} \widetilde \bM^T x|\\
&\leq |y-x| + |\widetilde\bM(\widetilde \bM^T \widetilde\bM)^{-1} (\widetilde \bM^T-\widetilde\bA^T)x|,\\
& \mbox{ using (\ref{norm-100}) and } x^T\widetilde \bA=0\\
&\leq |y-x|+\frac{1}{s_r(\widetilde \bM)}||\widetilde \bM-\widetilde \bA||\\
& \mbox{ since }||\widetilde\bM(\widetilde \bM^T \widetilde\bM)^{-1}||=\frac{1}{s_r(\widetilde\bM)}\\
&\leq\frac{\alpha^2}{500k^{8.5}}+\frac{k^{4.5}\sigma}{\alpha\sqrt\delta s_k(\bM)}, \mbox{ using (\ref{x-y} and \ref{1892})}.
\end{align*}
$|x-z|\leq |x-y|+|y-z|$ and using Claim (\ref{singular-P-M}), the first assertion of the Lemma follows.

To prove (\ref{sin-theta-M-U}), we argue that Dim$(U)= k-r$ (this plus (\ref{cos-F-F-prime}) proves (\ref{sin-theta-M-U}).)
$U$ has dimension at least $k-r$. If the dimension of $U$ is greater than $k-r$, then there is an orthonormal set of
$k-r+1$ vectors $u_1,u_2,\ldots ,u_{k-r+1}\in U$. By the first assertion, there are $k-r+1$ vectors
$w_1,w_2,\ldots ,w_{k-r+1}\in \Span(\bM)\cap \Null(\widetilde \bM)$ with $|w_t-u_t|\leq\delta _3, t=1,2,\ldots ,k-r+1$.
For $t\not= t'$, we have
$$|w_t\cdot w_{t'}|\leq |u_t\cdot u_{t'}|+|(w_t-u_t)\cdot u_{t'}|+|w_t\cdot (w_{t'}-u_{t'})|\leq 2\delta_3.$$
So the matrix $(w_1|w_2|\ldots |w_{k-r+1})^T \; (w_1|w_2|\ldots |w_{k-r+1})$ is diagonal-dominant and therefore non-singular.
So, $w_1,w_2,\ldots ,w_{k-r+1}$ are linearly independent vectors
in $\Span(\bM)\cap\Null(\widetilde \bM)$ which contradicts the fact that the dimension of $\Span(\bM)\cap\Null(\widetilde \bM)$
is $k-r$.
This finishes the proof of Lemma ~\ref{VcapNull-Mnull}
\hfill\qed
\end{proof}

We have (using ~\ref{A-M} and Cauchy-Schwartz inequality):
\begin{equation}\label{1892}
||\widetilde \bM-\widetilde \bA||\leq\Max_{w:|w|=1}\left| (\widetilde \bM-\widetilde \bA)w\right|\le \frac{k^{4.5}}{\alpha}\frac{\sigma}{\sqrt\delta}.
\end{equation}

\begin{claim}\label{q-ell-length}
If $\ell,\ell'\notin \{ \ell_1,\ell_2,\ldots ,\ell_r\}$, $\ell\not= \ell'$, then,
\begin{equation}\label{ell-ell-prime-proj}
|\mbox{proj} (M_{\cdot,\ell}-M_{\cdot,\ell'},\mbox{Null}(\widetilde \bM))|
\geq \alpha \Max_{\ell''} |M_{\cdot,\ell''}|.
\end{equation}
\end{claim}
\begin{proof}
$|\mbox{proj}(M_{\cdot,\ell}-M_{\cdot,\ell'},Null(\widetilde\bM))$
$=\Min_x |M_{\cdot,\ell}-M_{\cdot,\ell'}-\widetilde \bM x| $
$\geq \Min_{\beta, x} |M_{\cdot,\ell}-\beta M_{\cdot,\ell'}-\widetilde\bM x|$
$$\geq \min_{y\in {\bf R}^{k-1}}
|M_{\cdot,\ell}-\sum_{\ell''\not=\ell}y_{\ell''}M_{\cdot,\ell''}|$$
$$= |\mbox{proj} (M_{\cdot,\ell},\mbox{Null}(\bM\setminus M_{\cdot,\ell}))|\geq \alpha \Max_{\ell''} |M_{\cdot,\ell''}|,
$$
where, the last inequality is from (\ref{condition-number-100}). \hfill
\end{proof}

Next, we prove the Lemma that states that $|u\cdot x|$ has an unambiguous optimum over $K$:
I..e., there is an $\ell$ so that $|u\cdot M_{\cdot,\ell}|$ is a definite amount higher
than any other $|u\cdot M_{\cdot,\ell'}|$.
 The reader may want to consult
the intuitive description in Section (\ref{first-step}) for the role played by this Lemma in the proof
of correctness. In short, this Lemma would say that if we were able to optimize over $K$, we could get a
hold of a new vertex. While this may first seem tautological, the point is that if there were
ties for the optimum over $K$, then, instead of a vertex, we may get a point in the interior of a
face of $K$. Indeed, since the sides of $K$ are relatively small (compared to $n,d$), it requires some
work (this lemma) to rule this out.
This alone is not sufficient, since we have access only to $K'$, not $K$. The next Lemma will prove that
the optimal solutions (not just solution values) over $K$ and $K'$ are close.
\begin{lemma}\label{u-good}
Let $u$ be as in the algorithm. With probability at least $1-(c/k^{3/2})$, the following hold:
\begin{align*}
\forall \ell,\ell' \notin \{ \ell_1,\ell_2,\ldots ,\ell_r\}, \ell\not=\ell': |u\cdot (M_{\cdot,\ell}-M_{\cdot,\ell'})| \\
\geq \frac{.097}{k^{4}}\alpha \Max_{\ell''} |M_{\cdot,\ell''}|.\\
\forall \ell\notin \{ \ell_1,\ell_2,\ldots ,\ell_r\}: |u\cdot (M_{\cdot,\ell})|\geq \frac{.09989\alpha }{k^{4}} \Max_{\ell''} |M_{\cdot,\ell''}|.
\end{align*}
\end{lemma}

\begin{proof}
We can write
\begin{align*}
&M_{\cdot,\ell}=\underbrace{\mbox{Proj}(M_{\cdot,\ell},\mbox{Null}(\widetilde \bM))}_{q_\ell}+\underbrace{\mbox{Proj}(M_{\cdot,\ell},\mbox{Span}(\widetilde \bM))}_{p_\ell=\widetilde\bM w^{(\ell)}},
\end{align*}
where we use the fact that  $q_\ell$ can be written as $M_{\cdot,\ell}-\widetilde \bM w^{(\ell)}$ for some $w^{(\ell)}$.

From (\ref{condition-number-100}), we have
]$|q_\ell|\geq \alpha \Max_{\ell''} |M_{\cdot,\ell''}|$.
Since $|p_\ell|\leq |M_{\cdot,\ell}|$, and $s_r(\widetilde\bM)=\Min_{|x|=1}|\widetilde \bM x|\geq \Min_{|y|=1} |\bM y|=s_k(\bM)$, Claim (\ref{singular-P-M}) implies:
\begin{equation}\label{w-ell}
|w^{(\ell)}|\leq |p_\ell|/s_r(\widetilde \bM)\leq \frac{|M_{\cdot,\ell}|\alpha^2}{1000k^{8.5}}\frac{\sqrt\delta}{\sigma}.
\end{equation}

Recall $u$ in the Theorem statement - $u$ is a random unit length vector in subspace $U$.
$$u\cdot M_{\cdot,\ell}=u\cdot q_\ell +u^T\widetilde \bM w^{(\ell)}=  u\cdot \proj(q_\ell,U)
+u^T(\widetilde \bM -\widetilde \bA)w^{(\ell)},$$ since $u^T\widetilde\bA=0.$
So,
\begin{eqnarray}
 \left| u\cdot M_{\cdot,\ell}-u\cdot \proj(q_\ell,U)\right|&\leq ||(\widetilde \bM-\widetilde \bA)w^{(\ell)}|| \nonumber \\
 \leq ||\widetilde \bM-\widetilde\bA|| &\; |w^{(\ell)}|\leq \frac{|M_{\cdot,\ell}|\alpha}{1000k^4},\label{1755},
\end{eqnarray}
using (\ref{1892}) and (\ref{w-ell}). Similarly,
for $\ell'\not=\ell$.
$u\cdot(M_{\cdot,\ell}-M_{\cdot,\ell'})= u\cdot\proj(q_\ell-q_{\ell'},U)+u^T\widetilde\bM (w^{(\ell)}-w^{(\ell')}) $
\mbox{So, }
$|u\cdot(M_{\cdot,\ell}-M_{\cdot,\ell'})- u\cdot\proj(q_\ell-q_{\ell'},U)|\leq |u^T(\widetilde \bM-\widetilde\bA)(w^{(\ell)}-w^{(\ell')})|$
(using $u^T\widetilde A=0$)
\begin{equation}
\leq ||\widetilde\bM-\widetilde \bA|||w^{(\ell)}-w^{(\ell')}|\leq \frac{\alpha|M_{\cdot,\ell}-M_{\cdot,\ell'}|}{1000k^4}\label{1756},
\end{equation}
using (\ref{1892}) and
$|w^{(\ell)}-w^{(\ell')}|\leq |\widetilde\bM (w^{(\ell)}-w^{(\ell')})|/s_k(\bM)\leq |M_{\cdot,\ell}-M_{\cdot,\ell'}|/s_k(\bM)$,
since, $\widetilde \bM(w^{(\ell)}-w^{(\ell')})$ is an orthogonal projection of $M_{\cdot,\ell}-M_{\cdot,\ell'}$ into $\Span(\widetilde \bM)$
(and using Claim ~\ref{singular-P-M}).

Now, $u$ is a random unit length vector in $U$. Now, $\proj(q_\ell ,U), \proj(q_\ell -q_{\ell'},U), \ell,\ell'\in [k]$ are fixed
vectors in $U$ (and the choice of $u$ doesn't dependent on them).
 Consider the following event ${\cal E}$:
$${\cal E}: \forall \ell : |u\cdot \proj(q_\ell ,U)|\geq \frac{1}{10k^{4}} |\proj(q_\ell,U)|\mbox{  AND  }$$
$$\forall \ell\not= \ell': |u\cdot \proj(q_\ell -q_{\ell'},U)|\geq \frac{1 }{10k^{4}} |\proj(q_\ell-q_{\ell'},U)|. $$
The negation of ${\cal E}$ is the union of at most $k^2$ events (for each $\ell$ and each $\ell,\ell'$) and each of these has a failure
probability of at most $1/10k^{3.5}$ (since the $k-1$ volume of $\{x\in U: u\cdot x=0\}$ is at most
$\sqrt k$ times the volume of the unit ball in $U$). Thus, we have:
\begin{equation}\label{1855}
\Prob ({\cal E})\geq 1-\frac{1}{10k^{1.5}}.
\end{equation}
We pay the failure probability and assume from now on that ${\cal E}$ holds.


By (\ref{sin-theta-M-U}), we have that there is a $q'_\ell\in U$ with $|q'_{\ell}-q_\ell|\leq  \alpha |q_\ell|/(100k^4)$ which
implies (recall $k\geq 2$):
\begin{equation}\label{q-ell-proj-U}
|q_\ell-\proj(q_\ell,U)|\leq \frac{\alpha}{100k^4}|q_\ell|\leq \frac{|q_\ell|}{1600}\;\implies\; |\Proj(q_\ell,U)|\geq .9999|q_\ell|.
\end{equation}
So, under ${\cal E}$,
\begin{align}\label{979}
&\forall \ell\notin \{\ell_1,\ell_2,\ldots ,\ell_r\},  |u\cdot \proj(q_\ell,U)|\geq |\proj(q_\ell,U)|\frac{1}{10k^{4}}\geq\frac{.09999|q_\ell|}{k^4}\geq \frac{.09999\alpha \Max_{\ell''} |M_{\cdot,\ell''}|}{k^4},
\end{align}
since $|q_\ell|\geq |proj(\bM_{\cdot,\ell},\Null(\bM\setminus M_{\cdot,\ell}))|\geq \alpha \Max_{\ell''} |M_{\cdot,\ell''}|$ by (\ref{condition-number-100}).

By (\ref{1755}) and (\ref{979}), $\forall \ell\notin \{\ell_1,\ell_2,\ldots ,\ell_r\}$,
\begin{equation}\label{555}
|u\cdot M_{\cdot,\ell}|\geq |u\cdot \Proj(q_\ell,U)|-\frac{\alpha|M_{\cdot,\ell}|}{1000k^4}\geq \frac{.09989\alpha \Max_{\ell''} |M_{\cdot,\ell''}|}{k^4},
\end{equation}
proving the second assertion of the Lemma.

Now we prove the first assertion. For $\ell\notin \{ \ell_1,\ell_2,\ldots ,\ell_r\}$ and $\ell'\notin \{ \ell, \ell_1,\ell_2,\ldots ,\ell_r\}$, by (\ref{1756}),
$$|u\cdot (M_{\cdot,\ell}-M_{\cdot, \ell'})|  \geq |u\cdot\proj(q_\ell -q_{\ell'},U)|-  \frac{\alpha|M_{\cdot,\ell}-M_{\cdot,\ell'}|}{1000k^4}$$
$$\geq \frac{1}{10k^4}|\Proj(q_\ell-q_{\ell'},U)|-\frac{\alpha|M_{\cdot,\ell}-M_{\cdot,\ell'}|}{1000k^4}\quad\text{ by }{\cal E},$$
since, by ~\ref{sin-theta-M-U},  $\exists x\in U,$: $|x-(q_\ell-q_{\ell'})|\leq \frac{\alpha |q_\ell-q_{\ell'}|}{100k^4}$, we have
$|\Proj(q_\ell-q_{\ell'},U)|\geq .99 |q_\ell-q_{\ell'}|\geq .99 \alpha\Max_{\ell''}|M_{\cdot,\ell''}|,$ by Claim ~\ref{q-ell-length}.
This finishes the proof of the first assertion and of the Lemma.
\hfill
\end{proof}
We just proved that $|u\cdot x|$ has an unambiguous maximum over $K$. The following Lemma shows that if $M_{\cdot,\ell}$ is this optimum,
and if $A_{\cdot,S}$ is the optimum of $|u\cdot x|$ over $K'$, then, $A_{\cdot, S}\approx M_{\cdot,\ell}$.
The idea of the proof is that for the optimal $A_{\cdot,S}$, the corresponding $P_{\cdot,S}$ which is in $K$
is a convex combination of all columns of $\bM$. If the convex combination involves any appreciable
amount of non-optimal vertices of $K$, since, by the last Lemma, $|u\cdot x|$ is considerably less at non-optimal
vertices than the optimal one, $|u\cdot A_{\cdot,S}|$ would be considerably less than $|u\cdot M_{\cdot,\ell}|$, where,
$M_{\cdot,\ell}$ is the optimum over $K$. This produces a contradiction to $A_{\cdot,S}$ being optimal over $K'$
since, by (\ref{extreme-docs}), there is a set $S_\ell$ with $|u\cdot A_{\cdot,S_\ell}|\approx |u\cdot M_{\cdot,\ell}|$.

\begin{lemma}\label{S-good}
Let $R_{r+1}$ be an in algorithm. Define $\ell$ by:
$$\ell=\begin{cases} \arg\max_{\ell'} u\cdot M_{\cdot,\ell'} &\mbox{ if }u\cdot A_{\cdot,R_{r+1}}\geq 0\\
\arg\min_{\ell'} u\cdot M_{\cdot,\ell'} &\mbox{ if }u\cdot A_{\cdot,R_{r+1}}<0\end{cases}.$$
Then, under the hypothesis of Theorem ~\ref{main-technical-theorem}
$\ell\notin\{ \ell_1,\ell_2,\ldots ,\ell_r\}$ and $$|A_{\cdot,R_{r+1}}-M_{\cdot,\ell}|\leq \frac{150k^4}{\alpha}\frac{\sigma}{\sqrt\delta}.$$
\end{lemma}
\begin{proof}

{\bf Case 1} $u\cdot A_{\cdot, R_{r+1}}\geq 0$.

We scale $u$, so that $|u|=1$ which does not change $R_{r+1}$ found by the
algorithm.
Now,
$$\ell=\arg\max_{\ell'} u\cdot M_{\cdot,\ell'}.$$
We claim that $\ell\notin \{ \ell_1,\ell_2,\ldots ,\ell_r\}$. Suppose for contradiction,
$\ell\in\{ \ell_1,\ell_2,\ldots ,\ell_r\}$; wlg, say $\ell=\ell_1$. Then,
by the hypothesis of Theorem ~\ref{main-technical-theorem}, we have that $|A_{R_1}-M_{\cdot,\ell_1}|\leq 150\sigma k^4/\alpha\sqrt\delta$ and so,
$u\cdot M_{\cdot,\ell_1}
\leq u\cdot A_{\cdot,R_1}+(150k^4\sigma)/(\alpha\sqrt\delta)=(150k^4\sigma)/(\alpha\sqrt\delta)$ (since $u\in U$ and so $u\perp A_{\cdot,R_1}$).
So, for all $\ell'$,
$u\cdot M_{\cdot, \ell'}\leq u\cdot M_{\cdot,\ell_1}\leq (150k^4\sigma)/(\alpha\sqrt\delta)$.
So, for all $R\subseteq [n]$, $P_{\cdot,R}$ which is in CH$(\bM)$, satisfies $u\cdot P_{\cdot,R}\leq 150k^4\sigma/\alpha\sqrt\delta$.
So, by Lemma (\ref{averages}), $u\cdot A_{\cdot,R_{r+1}}\leq u\cdot P_{\cdot, R_{r+1}}+(\sigma/\sqrt\delta)\leq ((150k^4/\alpha)+1)\sigma/\sqrt\delta$. But for
any $t\notin \{ \ell_1,\ell_2,\ldots ,\ell_r\}$,
we have with $S_t$ as in \ref{extreme-docs},
\begin{align*}
&|u\cdot A_{S_t}|\geq|u\cdot P_{\cdot, S_t}|-
(\sigma/\sqrt\delta)\text{ Lemma (\ref{averages})}\\
&\geq |u\cdot M_{\cdot,t}|-(5\sigma/\sqrt\delta)\geq .09989\alpha\Max_{\ell''}|M_{\cdot,\ell''}|/(k^4)-5\sigma/\sqrt\delta\text{   Lemma ~\ref{u-good} and \ref{extreme-docs}.  }
\end{align*}
and so, $u\cdot A_{\cdot,R_{r+1}}$
(which maximizes $u\cdot A_{\cdot,R}$ over all $R, |R|=\delta n$)
must be at least $\frac{\alpha\Max_{\ell''}|M_{\cdot,\ell''}|}{11k^4}-\frac{5\sigma}{\sqrt\delta}$
contradicting $u\cdot A_{\cdot, R_{r+1}}\leq ((150k^4/\alpha)+1)\sigma/\sqrt\delta$ by ~\ref{spectral-bound}. So, $\ell\notin\{ \ell_1,\ell_2,\ldots ,\ell_r\}$
and by Lemma (\ref{u-good}),
\begin{equation}\label{1257}
u\cdot M_{\cdot,\ell}\geq \frac{.09989\alpha\Max_{\ell''}|M_{\cdot,\ell''}|}{k^{4}}.
\end{equation}
We have $|P_{\cdot,j}-M_{\cdot,\ell}|\leq \frac{4\sigma}{\sqrt\delta}$ for all $j\in S_\ell$, so
also $|P_{\cdot,S_\ell}-M_{\cdot,\ell}|\leq \frac{4\sigma}{\sqrt\delta}$.
\begin{equation}
u\cdot A_{\cdot,S_\ell}\geq u\cdot P_{\cdot,S_\ell}-\frac{\sigma}{\sqrt\delta}\geq u\cdot M_{\cdot,\ell}-\frac{5\sigma}{\sqrt\delta}.
\end{equation}
By the definition of $R_{r+1}$,
\begin{equation}\label{u-cdot-A-S}
u\cdot A_{\cdot,R_{r+1}}\geq u\cdot A_{\cdot,S_\ell}\geq u\cdot M_{\cdot,\ell}-\frac{5\sigma}{\sqrt\delta},
\end{equation}
For any $\ell'\notin \{ \ell,\ell_1,\ell_2,\ldots ,\ell_r\}$, we have by Lemma ~\ref{u-good},
$$u\cdot M_{\cdot,\ell'}\leq u\cdot M_{\cdot,\ell}-\frac{.097\alpha}{k^{4}}\Max_{\ell''}|M_{\cdot,\ell''}|.$$
Also, for $\ell'\in\{ \ell_1,\ell_2,\ldots ,\ell_r\}$, wlg, say $\ell'=\ell_1$, we have noting that $|A_{\cdot,R_1}-M_{\cdot,\ell_1}|\leq 150k^4\sigma/(\alpha \sqrt\delta)$
from the hypothesis of Theorem ~\ref{main-technical-theorem}:
\begin{align*}
&u  \cdot M_{\cdot,\ell_1}\leq u\cdot A_{\cdot, R_1}+150k^4\sigma/\alpha\sqrt\delta=150k^4\sigma/\alpha\sqrt\delta\\
&\leq u\cdot M_{\cdot,\ell}-\frac{.09989\alpha\Max_{\ell''}|M_{\cdot,\ell''}|}{k^{4}}+\frac{150k^4\sigma}{\alpha\sqrt\delta}\text{ by ~\ref{1257}} \\
&\leq u\cdot M_{\cdot,\ell}-\frac{.097\alpha\Max_{\ell''}|M_{\cdot,\ell''}|}{k^4}
\end{align*}
Now, $P_{\cdot,R_{r+1}}$ is a convex combination of the columns of $\bM$; say the convex combination is $P_{\cdot,R_{r+1}}=\bM w$.
From above, we have:
\begin{align*}
&u\cdot A_{\cdot, R_{r+1}}\leq u\cdot P_{\cdot,R_{r+1}}+\frac{\sigma}{\sqrt\delta}\\
&\leq w_\ell (u\cdot M_{\cdot,\ell})+\sum_{\ell'\not=\ell}\left( (u\cdot M_{\cdot,\ell})-\frac{.097\alpha}{k^4}\Max_{\ell''}|M_{\cdot,\ell''}|\right)w_{\ell'}\\
&\leq u\cdot M_{\cdot,\ell}-\frac{.097\alpha}{k^4}\Max_{\ell''}|M_{\cdot,\ell''}| (1-w_\ell).
\end{align*}

This and (\ref{u-cdot-A-S}) imply:
\begin{equation}\label{1020}
(1-w_\ell)\Max_{\ell''}|M_{\cdot,\ell''}|\leq \frac{52k^{4}}{\alpha}\frac{\sigma}{\sqrt\delta}.
\end{equation}
So,
$|P_{\cdot,R_{r+1}}-M_{\cdot,\ell} | = \left| (w_\ell-1)M_{\cdot,\ell}+\sum_{\ell'\not=\ell} w_{\ell'}M_{\cdot,\ell'}\right|$
$$\leq \sum_{\ell'\not=\ell}w_{\ell'}|M_{\cdot,\ell}-M_{\cdot,\ell'}|\leq 2(1-w_\ell)\Max_{\ell''}|M_{\cdot,\ell''}|\leq\frac{104k^4}{\alpha}\frac{\sigma}{\sqrt\delta}.$$
Now it follows that $|A_{\cdot,R_{r+1}}-M_{\cdot,\ell} |\leq \frac{150k^4}{\alpha}\frac{\sigma}{\sqrt\delta}$
finishing the proof of the theorem in this case.
An exactly symmetric argument proves the theorem in the case when $u\cdot A_{\cdot,S}\leq 0.$
\hfill
\end{proof}


{\bf Time Complexity}

Our algorithm above is novel in the sense this approach of
using successive optimizations to find extreme points of the hidden simplex does not seem to be used in any of the special cases.
It also has a more useful consequence: we are able to show that the only way we treat $\bA$ is matrix-vector
products and therefore we are able to prove a running time bound of $O^*(k\text{ nnz}+k^2d)$ on the algorithm.
We also use the observation that the SVD at the start can be done in the required time
by the classical sub-space power method. The SVD as well as keeping track of the subspace $W$ are done by
keeping and updating a $d\times (k-r)$ (possibly dense) matrix whose columns form a basis of
$W$. We note that one raw iteration of the standard $k-$means algorithm finds the distance between each of
$n$ data points and each of $k$ current cluster centers which takes $O(k\text{nnz})$ time matching the leading term of our
total running time.

The first step of the algorithm is to do $O(\ln d)$ subspace-power iterations. Each iteration
starts with a $d\times k$ matrix ${\bf Q_t}$ with orthonormal columns, multiplies $\bA\bA^T{\bf Q_t}$ and makes the
product's columns orthonormal by doing a Grahm-Schmidt. The products (first pre-multiply ${\bf Q_t}$ by $\bA^T$ and then
pre-multiply by $\bA$ take time $O(\text{nnz})$. Doing Grahm-Scmidt takes involves dot product of each column with previous
columns and subtracting out the component. The columns of ${\bf Q_{t+1}}$ are possibly dense, but, still, each dot product
takes time $O(d)$ and there are $k^2$ of them for a total of $O(dk^2)$ per iteration times $O^*(1)$ iterations.

The rest of the algorithm has the complexity we claim. We do $k$ rounds in each of
which, we must first choose a random $u\in V\cap\Null(A_{\cdot,R_1},A_{\cdot,R_2},\ldots A_{\cdot,R_r})$.
To do this, we keep a orthonormal basis of $\Span  (A_{\cdot,R_1},A_{\cdot,R_2},\ldots A_{\cdot,R_r})$; updating this
once involves finding the dot product of $A_{\cdot, R_{r+1}}$ with the previous basis in time $O(dk)$, for a total
of $dk^2$.
Now to pick a random $u\in V\cap\Null(A_{\cdot,R_1},A_{\cdot,R_2},\ldots A_{\cdot,R_r})$, we just pick
a random $u$ from $V$ and then
subtract out  its component in
$\Span(A_{\cdot, S_1},A_{\cdot,S_2},\ldots , A_{\cdot,S_r})$. All of this can be done in $O^*(k\; \text{nnz}(\bA)+k^2d)$ time.
This completes the proof of Theorem \ref{main-technical-theorem}.
\end{proof}

\section{Conclusion}
The dependence of the Well-Separatedness on $k$ could be improved.
For Gaussian Mixture Models, one can get $k^{1/4}$, but this is a very special case of our problem. But in any case,
something substantially better than $k^8$ would seem reasonable to aim for.
Another important improvement of the same assumption would be to ask only that each
column of $\bM$ be separated in distance (not in perpendicular component) from the others.
An empirical study of the speed and quality of solutions of this algorithm in comparison to
Algorithms for special cases would be an interesting story of how well asymptotic complexity
reflects practical efficacy in this case.
The subset-soothing construction should be applicable to other models where there is stochastic
Independence, since subset averaging improves variance in general.

\bibliographystyle{plainnat}
\bibliography{polytope,jmlr_bib,nmf}
\newpage

\end{document}